\documentclass{article}

\PassOptionsToPackage{numbers}{natbib}
 \usepackage[preprint]{neurips_2026}

\usepackage[utf8]{inputenc}
\usepackage[T1]{fontenc}
\usepackage{hyperref}
\usepackage{url}
\usepackage{booktabs}
\usepackage{amsfonts}
\usepackage{pifont}
\usepackage{nicefrac}
\usepackage{microtype}
\usepackage{algorithm}
\usepackage{algorithmic}
\usepackage{xcolor}
\usepackage{graphicx}
\usepackage{multirow}
\usepackage{colortbl}
\usepackage{titletoc}
\usepackage{enumerate}
\usepackage{enumitem}
\usepackage{subfigure}
\usepackage{wrapfig}
\usepackage{caption}
\usepackage{xurl}

\usepackage{amsmath}
\usepackage{amssymb}
\usepackage{mathtools}
\usepackage{amsthm}

\usepackage[capitalize,noabbrev]{cleveref}

\theoremstyle{plain}
\newtheorem{theorem}{Theorem}[section]

\newtheorem{lemma}[theorem]{Lemma}
\newtheorem{corollary}[theorem]{Corollary}
\theoremstyle{definition}

\theoremstyle{remark}

\usepackage[disable,textsize=tiny]{todonotes}

\setlist[enumerate]{itemsep=2pt, parsep=1pt, topsep=1pt}

\title{Can Subgraph Explanations Be Weaponized to Steal Graph Neural Networks?}

\author{%
  Ojas Nimase \\
  University of Southern California\\
  Los Angeles, USA \\
  \texttt{nimase@usc.edu} \\
  \And
  Jiate Li \\
  University of Southern California\\
  Los Angeles, USA \\
  \texttt{jiateli@usc.edu} \\
  \And
  Yue Zhao \\
  University of Southern California\\
  Los Angeles, USA \\
  \texttt{yue.z@usc.edu} \\
  \And
  Yushun Dong \\
  Florida State University\\
  Tallahassee, USA \\
  \texttt{yushun.dong@fsu.edu} \\
}

\begin{document}

\maketitle

\begin{abstract}
Graph Machine Learning as a Service (GMLaaS) platforms increasingly implement explainability interfaces to meet regulatory transparency requirements. However, this transparency creates exploitable vulnerabilities for model extraction attacks. We present the first model extraction attack specifically designed for graph classification under strict black-box constraints where the attacker observes only discrete class labels and binary explanation masks (no probability scores, gradients, or confidence values). Our method (1) uses model explanation outputs to guide Monte Carlo edge sensitivity estimation toward decision boundaries, with Hoeffding concentration guarantees on estimation accuracy and (2) exploits explanation subgraphs to efficiently narrow the boundary search space. Extensive experiments on benchmark graph datasets across multiple domains demonstrate our method's superiority over comparable baselines. These findings demonstrate that such explainability interfaces create exploitable attack surfaces, informing both defensive mechanisms and policy frameworks for explainable AI mandates. The implementation code is provided in \url{https://github.com/LabRAI/XSTEAL/}. 
\end{abstract}

\section{Introduction}
\label{sec:intro}
Graph Machine Learning as a Service (GMLaaS) allows organizations to distribute access to proprietary graph-based machine learning models through publicly accessible APIs while keeping internal model workings secure~\citep{li2025intellectualpropertygraphbasedmachine}, widely applied in applications like fraud detection~\cite{aws_gnn_fraud_2022} and recommendation systems~\cite{10.1145/3535101, gnnie_2022}. However, this paradigm faces a fundamental security threat: model extraction attacks, which reconstructs functionally similar surrogates to replicate the victim's decision-making behavior by adversarially exploiting API access to query the inner models~\cite{10.5555/3241094.3241142, 10.1145/3595292, maslej2025artificialintelligenceindexreport,ZHOU202057}. This threat has motivated growing study across ML paradigms, spanning attacks, defenses, and ownership verification~\cite{shen2026creditcertifiedownershipverification,Li_Yuan_Shen_Le_Wang_Zhou_Gao_Dong_2026,zhao2025surveymodelextractionattacks}. Compounding this threat, growing regulatory requirements~\cite{euai2024, colorado_sb24205_2024} and user demand~\cite{Rosenbacke2024ExplainableAITrust} are pushing providers to further augment their APIs with explainability interfaces, where users would additionally be aware of the most influencing graph components%(nodes, edges, substructures) most influenced 
on the predicted labels~\cite{ying2019gnnexplainergeneratingexplanationsgraph}. While such transparency mechanisms help satisfy compliance requirements, they also expose a new attack channel that fundamentally threatens the security premise GMLaaS is built upon. This fundamental tension between mandated transparency and model security has been identified as an open challenge for MLaaS~\cite{11416086}. This raises a critical question: can the explanations mandated for transparency be weaponized to steal the same models they explain?

Recent work has begun exploiting this vulnerability, demonstrating that explanations can enable more efficient model extraction attacks~\cite{Oksuz_2024, 10.1145/3665451.3665533, ma2025explanationsleakdecisionlogic}. However, this emerging area remains sparsely studied, especially for graph neural networks. Current studies face three key limitations. \textbf{First}, existing GNN extraction methods lack a principled understanding of which training data is most effective for extraction. Current approaches construct query samples by centering on high-confidence components~\cite{Oksuz_2024, 10.1145/3665451.3665533, ma2025explanationsleakdecisionlogic}, but never formalize why these samples aid extraction or whether better strategies exist. Intuitively, a high-fidelity surrogate must replicate the victim's behavior not only in high-confidence regions but also near decision boundaries where predictions change. Previous work on general (non-graph) model extraction has empirically demonstrated the effectiveness of boundary sampling~\cite{10.1145/3689932.3694756}, yet this insight has not been extended to the GNN domain where the discrete, combinatorial nature of graph data introduces fundamentally different challenges. \textbf{Second}, boundary search techniques developed for continuous domains cannot transfer to discrete graph spaces. Existing methods utilizing boundary points~\cite{10.1145/3689932.3694756, davies2026boundary/b2} require input space continuity to efficiently locate decision boundaries. On graphs, these techniques fundamentally fail: given two graphs $G_1, G_2$ with different predicted labels, there is no unique ``midpoint'' as combinatorially many graphs exist at equal edit distance from both (Figure~\ref{fig:graph_median_visual})~\cite{10.14778/3594512.3594514}. Moreover, naively estimating edge sensitivities over all $O(n^2)$ possible edges incurs prohibitive query costs~\cite{jin2025queryefficientzerothorderalgorithmsnonconvex}, compounded by the hybrid discrete-continuous nature of graph data that prevents direct application of standard optimization techniques~\cite{dai2018adversarialattackgraphstructured, 10.7717/peerj-cs.2095}. \textbf{Third}, every existing GNN extraction method requires access to rich continuous information (probability scores, continuous importance vectors, or gradient-derived signals)~\cite{10.1145/3665451.3665533, Oksuz_2024, ma2025explanationsleakdecisionlogic, shen2021modelstealingattacksinductive, 10.5555/3698900.3699194} that realistic API deployments do not provide. Under emerging regulatory mandates~\cite{euai2024, colorado_sb24205_2024}, GMLaaS providers are increasingly required to furnish explanations alongside predictions. In practice, rational providers would expose the minimum information necessary for compliance (discrete class labels and binary explanation masks) rather than continuous scores or gradients that reveal additional proprietary model internals. No existing GNN extraction method operates under these minimal-information constraints.

\begin{wrapfigure}{l}{0.6\textwidth}
    \centering
    \vspace{-6pt}
    \includegraphics[width=0.59\textwidth]{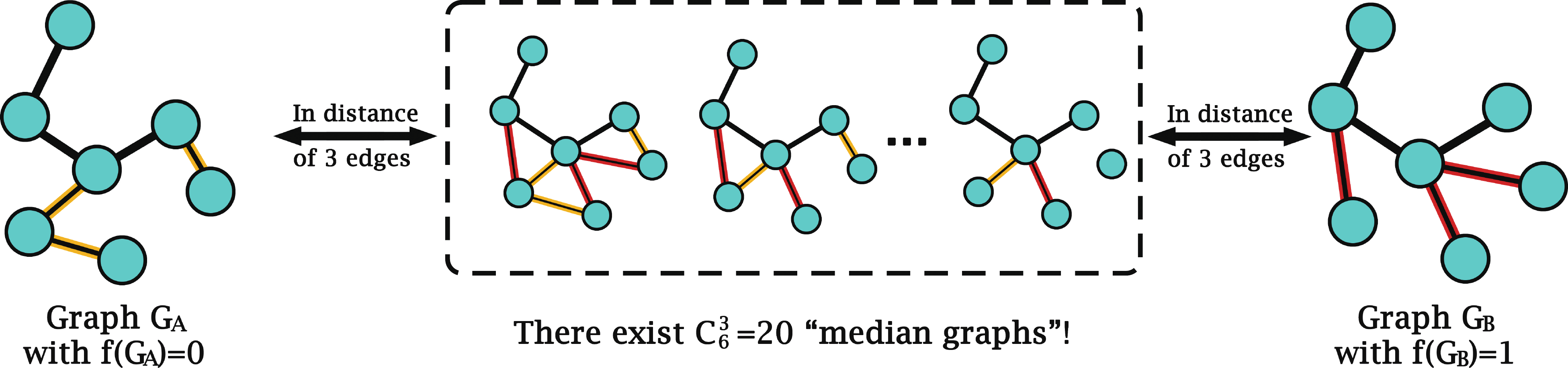}
    \vspace{-3pt}
    \caption{Multiple possible ``midpoint graphs'' at equal edit distance from $G_A$ and $G_B$.}
    \label{fig:graph_median_visual}
    \vspace{-11pt}
\end{wrapfigure}

To address these limitations, we (1) define and investigate the effectiveness of boundary samples for GNN extraction, formalizing the extraction problem as a boundary pair search. (2)~ Propose an explanation-guided search attack that locates such boundary samples under strict black-box constraints where the attacker observes only discrete predictions and binary explanation masks. Our search iteratively estimates edge sensitivities via Monte Carlo sampling from discrete label responses, flips the most sensitive edges, and verifies boundary crossings. We restrict candidates to edges within or adjacent to the explanation subgraph to reduce the search space from $O(|V|^2)$ to $O(|V_{\mathrm{sub}}|^2)$. We focus on graph classification, a setting that presents unique challenges for model extraction and has not been previously addressed under strict black-box constraints. Extensive experiments across molecular, visual, and synthetic graph benchmarks validate our framework. In summary, our contributions are:

\begin{itemize}[leftmargin=*,nosep] 
\item \textbf{Boundary Sample Effectiveness for GNNs}: We provide the first formalized definition of boundary samples for GNN extraction and empirically validate their superiority over random and high-confidence sampling strategies, recasting model extraction as a boundary pair search problem.
\item \textbf{Novel Explanation-Guided Extraction Attack}: Under strict black-box constraints with only discrete predictions and binary explanation masks, we develop the first model extraction attack for graph classification that combines a provably accurate estimator of discrete edge sensitivity with explanation-guided candidate reduction to efficiently generate boundary samples.
\item \textbf{Comprehensive Empirical Evaluation}: We conduct extensive experiments across eight graph datasets spanning molecular, visual, and synthetic domains, three victim architectures, and two explainer types, demonstrating consistent superiority over adapted baselines in surrogate-victim fidelity. 
\end{itemize}

\section{Preliminaries and Motivations}
\paragraph{Graph Neural Networks and Explanations.}
A graph $G = (V, E, X)$ consists of nodes $V$, edges $E \subseteq V \times V$, and node features $X \in \mathbb{R}^{|V| \times d}$. Graph neural networks (GNNs)~\cite{kipf2017semisupervisedclassificationgraphconvolutional, veličković2018graphattentionnetworks, hamilton2018inductiverepresentationlearninglarge} process $G$ via a message passing mechanism to obtain node-level embeddings. In graph classification tasks, the GNN, denoted as $f$, further applies a pooling operation for a graph-level embedding, and maps the embedding to a categorical label: $f: \mathcal{G} \to [C]$, where $[C] = \{1, \ldots, C\}$ denotes the label space.
Under the specific model $f$ and the graph $G$, a GNN explainer $h$ then returns an explanatory subgraph $G_{\mathrm{sub}} = (V_{\mathrm{sub}}, E_{\mathrm{sub}})$ where $V_{\mathrm{sub}} \subseteq V$ and $E_{\mathrm{sub}} \subseteq E$, identifying nodes and edges most influential for the prediction $f(G)$. Representative methods include GNNExplainer~\cite{ying2019gnnexplainergeneratingexplanationsgraph} and PGExplainer~\cite{NEURIPS2020_e37b08dd}.

\begin{figure*}[t]
    \centering
    \includegraphics[width=1\textwidth]{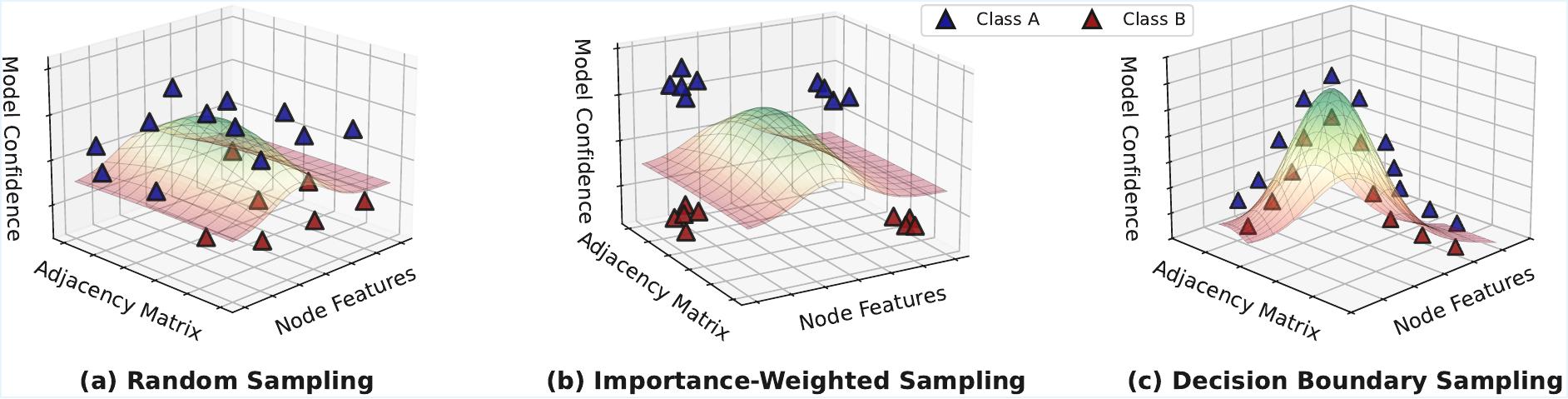}
    \vspace{-10pt}
    \caption{Sampling strategies for model extraction: surfaces are decision boundaries and triangles are classifications with smoothing for visualization clarity. Existing methods are (a) or (b), focusing on randomly sampling or sampling areas with high confidence. Our strategies are (a) and (c).}
    \label{fig:methodology_overview}
\end{figure*}

\paragraph{GNN Boundary Sampling Motivation.}
\label{subsec:motivation}
To better address the existing gap on extraction problem formulating and answer the question we raised in Section~\ref{sec:intro}: "What kind of sampling data is effective for extracting a GNN?", we investigate the effectiveness of boundary sampling on GNN extraction to solidify our motivation beforehand. We conduct a controlled experiment isolating the effect of boundary proximity on extraction fidelity across five binary benchmark datasets (AIDS, MUTAG, NCI1, PTC\_FM, Tox21\_AhR; dataset statistics are provided in Appendix~\ref{app:datasets}). Full experimental details are provided in Appendix~\ref{app:boundary-motivation}.

For each graph $G$ in a shadow dataset, we exhaustively test all single-edge toggles (adding or removing an edge) to identify \emph{boundary pairs} $(G, G')$ where $f(G) \neq f(G')$, and classify remaining graphs as \emph{non-boundary samples}. We compare four training strategies on equal-sized training sets: (1) Boundary: trained exclusively on boundary pair samples, (2) Non-boundary: trained on non-boundary samples only, (3) Hybrid: trained on 50\% boundary pairs + 50\% random samples from the full shadow set, and (4) Shadow: trained only on random samples from the full shadow set.

\begin{wrapfigure}{r}{0.5\textwidth}
    \centering
    \vspace{-12pt}
    \includegraphics[width=0.49\textwidth]{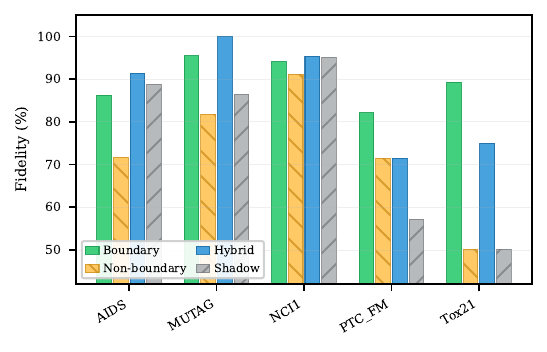}
    \vspace{-5pt}
    \caption{Fidelity (\%) by sampling strategy on equal-sized training sets. Full results in Appendix Table~\ref{tab:motivation-results}.}
    \label{fig:motivation-bar}
    \vspace{-24pt}
\end{wrapfigure}

Across all five datasets, surrogates with boundary data consistently outperform those without (Figure~\ref{fig:motivation-bar}), confirming that decision boundary information disproportionately aids model extraction. We provide full experimental details and a theoretical analysis in Appendix~\ref{app:boundary-motivation}.

\paragraph{Problem Formulation.} In this work we consider a \emph{strict black-box} model extraction attack against graph neural networks. Let $f: \mathcal{G} \to [C]$ denote a victim GNN classifier that maps graphs to discrete class labels $[C] = \{1, \ldots, C\}$. The attacker has query access only: given any graph $G$, the attacker can observe the predicted label $y(G) = f(G) \in [C]$, but cannot access probability scores, logits, confidence values, or internal model parameters. This strict assumption reflects realistic deployment scenarios where APIs return only classification decisions.

In general model extraction attack, the attacker's goal is to train a surrogate model $g$ replicating the decision behavior of $f$, which is however uneasy to be formalized in black-box setting. Motivated by the stated effectiveness of GNN boundary samples, we instead define our problem as constructing a query set $\bar{\mathcal{G}}$ consisting of \emph{boundary pairs}: pairs of similar graphs that receive different predictions. Formally, we aim to design an effective algorithm $M$ to find as much boundary pairs $(G, G')$ as possible, and replicate a model $g$ trained on the pairs to rebuild the same decision boundary:
\begin{equation}
\begin{aligned}
M(Q,f,h,\delta) &= \arg\max_{M} |\bar{\mathcal{G}}|, \bar{\mathcal{G}}=
\{(G,G'): f(G) \neq f(G')\wedge \operatorname{dis}(G, G') \leq \delta\}\\
g&=\arg \max_{g}\sum_{(G,G')\in\bar{\mathcal{G}}}\mathbb{I}(g(G)=f(G))+\mathbb{I}(g(G')=f(G'))
\label{eq:boundary-constraint}
\end{aligned}
\end{equation}
$Q$ is a budget that the adversary can afford to query on the victim model. $h$ is a GNN explainer on the holder side that returns a discrete explanatory subgraph $G_{\mathrm{sub}} = (V_{\mathrm{sub}}, E_{\mathrm{sub}})$ identifying nodes and edges important for the prediction $f(G)$. $\delta > 0$ is a distance threshold and $\operatorname{dis}(\cdot, \cdot)$ measures graph similarity, and the distance function is defined as:
$\operatorname{dis}(G, G') = \|A - A'\|_0$, i.e., the number of differing entries in the adjacency matrices. We focus on structural perturbation in this work, while our later experimental evaluation empirically demonstrates that edge flips alone are sufficient for constructing our boundary search. Feature perturbation can be incorporated if needed for feature-dominated tasks. It is worth noting that a smaller $\delta$ yields boundary pairs that more precisely characterize the local decision boundary geometry of $f$: a queried boundary pair forces every surviving hypothesis to place its decision boundary within an $O(\delta)$ neighborhood of the pair, as we prove in Lemma~\ref{lem:margin} (Appendix~\ref{sec:boundary-proof}).

\section{XSTEAL Framework}
\label{subsec:methodology}
In this section we describe how to construct our method to search boundary pairs $(G, G')$ satisfying the constraint in Eq.~(1) using only black-box queries to the victim GNN $f$ and its coupled explainer $h$. To begin with, we will firstly introduce up a warm-up algorithm that serves as the background of our designed method, and briefly introduce the arising challenges along with our corresponding technique solutions as an overview. Afterwards, we will formalize these components in details.

\paragraph{Warm-Up.} We notice that in convex machine learning on continuous data, it's easy to construct a binary search method: firstly, sampling two data $x$ and $x'$ with different prediction results; secondly, choose the middle data point $\frac{x+x'}{2}$ and make prediction on it; thirdly, replace $x$/$x'$ with the middle point that holds the same prediction, and check if the distance of the new point pair satisfies $\delta$. If not, we repeat from the first step on the new pair until we reach a satisfied boundary pair. 

\paragraph{Challenges and Solutions.} This simple binary search is efficient and ideal, however, facing a primary challenge when practically fitting into our problem: \textbf{(C1)} in graph learning domain, even we assume the GNN model is simple as a convex model, the discrete nature of the graph data makes it impossible to find proper middle point efficiently. To address this primary challenge, we propose \textbf{(S1)} a boundary search procedure that optimizes a sample data towards decision boundary guided by the data's local gradients. However, this requirement on gradients raised another challenge, that \textbf{(C2)} in our black-box extraction attack setting we are unable to obtain the local gradients on the victim model directly. Therefore, we accordingly propose \textbf{(S2)} a query-based Monte Carlo estimator on edge sensitivities, which approximates the local gradients by summarizing query results on the neighboring data. While this query-based estimation is provably accurate as an estimator of edge sensitivity, nevertheless, it requires certain cost on queries and calculations. \textbf{(C3)} Noticing in a graph data there are $\mathcal{O}(|V|^{2})$ edges, including existent and nonexistent ones, waiting to be estimated, the query burden in total makes the estimation extremely inefficient. To further address this efficiency challenge, we refine our method with \textbf{(S3)} an explanation-guided reduction on candidate edges for query efficiency. 

\subsection{Discrete Edge Sensitivity Estimation from Label Queries}
\label{subsubsec:grad-estimation}

\begin{figure*}[t]
    \centering
    \includegraphics[width=0.99\textwidth]{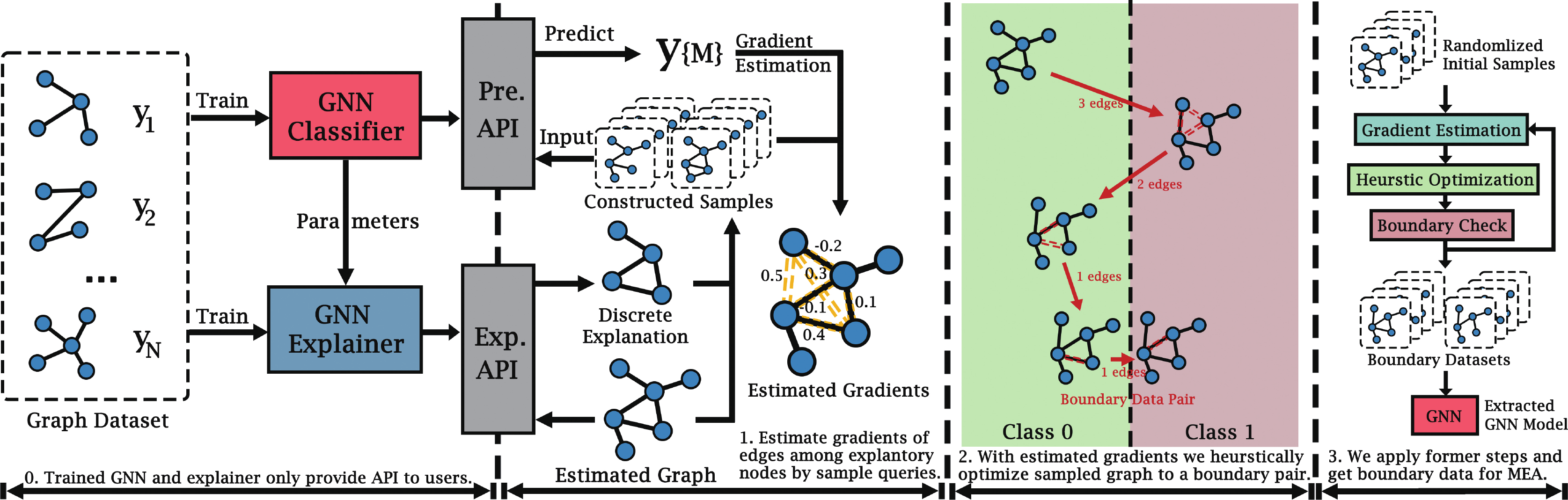}
    \caption{A visual depicting the setup of MEA and the pipeline of our proposed attack.}
    \label{fig:MAE_setup}
\end{figure*}

First, we introduce our estimator of edge sensitivity. Fix a base graph $G$ and let $G^{(t)}$ denote the current graph at iteration $t$ of our search procedure. We write $G^{(t)} = (A^{(t)}, X^{(t)})$ and $y^{(t)} \coloneqq y\bigl(G^{(t)}\bigr)$. Let $y^{(0)} = y(G)$ denote the original prediction of the unperturbed graph.

\paragraph{Black-box constraint.}
In contrast to white-box attacks that compute $\nabla_A \mathcal{L}$ directly, we have access only to the discrete prediction $y(\cdot) \in [C]$. We cannot observe probabilities, gradients, logits, or loss values. The same constraint applies to the explanation: the provider runs the explainer $h$ white-box on its own model to produce the mask, but the attacker receives only the resulting discrete node/edge set with no access to $h$'s gradients, importance scores, or internal parameters. This \textit{strict black-box} assumption reflects realistic deployment scenarios where APIs return only class labels and their accompanying discrete explanation masks.

\paragraph{Estimator of edge sensitivity.}
To estimate the sensitivity of edge $e$ without access to continuous outputs, we employ a Monte Carlo sampling approach. The key insight is that an edge's contribution to pushing the prediction away from $y^{(0)}$ can be estimated by observing how often the prediction changes when that edge is present versus absent across random perturbations of the graph.

For each candidate edge $e = (i,j) \in \mathcal{C}(G^{(t)})$, we define the graph obtained by flipping that edge:
\begin{equation}
    G^{(t)} \oplus e \;=\; \bigl(A^{(t)} \oplus e, X^{(t)}\bigr),
    \label{eq:edge-flip}
\end{equation}
where $A^{(t)} \oplus e$ is obtained from $A^{(t)}$ by setting $A^{(t)}_{ij} \gets 1 - A^{(t)}_{ij}$ and $A^{(t)}_{ji} \gets 1 - A^{(t)}_{ji}$.

For the $m$-th Monte Carlo sample, we generate a random perturbation $\xi_m$ applied to edges other than $e$. Let $\tilde{G}^{(t)}_m$ denote the perturbed graph with edge $e$ in its original state, and $\tilde{G}^{(t)}_m \oplus e$ denote the same perturbation but with edge $e$ flipped. We define the binary observation:
\begin{equation}
    Z_{e,m}^{(t)}
    \;=\;
    \mathbf{1}\bigl[y(\tilde{G}^{(t)}_m \oplus e) \neq y^{(0)}\bigr]
    - \mathbf{1}\bigl[y(\tilde{G}^{(t)}_m) \neq y^{(0)}\bigr],
    \label{eq:mc-observation}
\end{equation}
where $\mathbf{1}[\cdot]$ is the indicator function. Each $Z_{e,m}^{(t)} \in \{-1, 0, +1\}$ measures whether flipping edge $e$ pushes the prediction away from the original label ($+1$), back toward it ($-1$), or has no effect ($0$), under perturbation $\xi_m$. This class-agnostic formulation naturally handles both binary and multi-class settings: for binary classification it is equivalent to targeting the opposite class, while for multi-class it rewards any boundary crossing without privileging a specific target class. This spreads search across all $\binom{C}{2}$ pairwise boundaries which is the fragmentation effect analyzed in Section~\ref{subsec:main-results} (RQ3), for which targeted per-transition search is the natural refinement.

We estimate the edge sensitivity by the empirical mean over $n_e$ samples:
\begin{equation}
    \hat{g}_e^{(t)}
    \;=\;
    \frac{1}{n_e} \sum_{m=1}^{n_e} Z_{e,m}^{(t)}.
    \label{eq:empirical-sensitivity}
\end{equation}

A positive $\hat{g}_e^{(t)}$ indicates that flipping edge $e$ tends to push the prediction away from $y^{(0)}$ toward a decision boundary, while a negative value suggests it reinforces the current prediction. The magnitude reflects the strength of this effect.

\paragraph{Perturbation scheme.}
Each perturbation $\xi_m$ is generated by independently flipping each edge $(i',j') \neq e$ with probability $p_{\mathrm{flip}} \in (0, 0.5)$. This creates diverse graph neighborhoods around $G^{(t)}$ while preserving the overall structure. When $p_{\mathrm{flip}} = 0$, we recover the deterministic finite-difference estimate. Larger values are more robust to local decision boundary geometry, however at the cost of increased variance.

\paragraph{Concentration guarantees.}
Since each $Z_{e,m}^{(t)} \in [-1, +1]$, Hoeffding's inequality provides a concentration bound for our estimator.

\begin{theorem}[Hoeffding bound for edge sensitivities]
\label{lem:hoeffding}
Fix an iteration $t$ and an edge $e \in \mathcal{C}(G^{(t)})$. Let $g_e^{(t)} = \mathbb{E}[Z_{e,m}^{(t)}]$ denote the true expected sensitivity and $\hat{g}_e^{(t)}$ its empirical estimate. Then for any $\epsilon > 0$,
\begin{equation}
    \Pr\Bigl( \bigl|\hat{g}_e^{(t)} - g_e^{(t)}\bigr| \ge \epsilon \Bigr)
    \;\le\; 2 \exp\Bigl( -\frac{n_e \epsilon^2}{2} \Bigr).
    \label{eq:hoeffding}
\end{equation}
Equivalently, for any confidence level $1-s \in (0,1)$, with probability at least $1-s$:
\begin{equation}
    \bigl|\hat{g}_e^{(t)} - g_e^{(t)}\bigr|
    \;\le\;
    \sqrt{\frac{2 \log(2/s)}{n_e}}.
    \label{eq:hoeffding-confidence}
\end{equation}
\end{theorem}
\begin{proof}(Sketch) The binary observations $Z_{e}^{t}$ we made under the Monte Carlo sampling are independent and all satisfy $Z_{e,m}^{t}\in [-1,1],\forall m\in[1,n_{e}]$. Therefore by substituting the range $[-1,1]$ and $n_{e}$ into Hoeffding inequality~\cite{hoeffiding}, we could directly reach above equations.
\end{proof}
The estimation error decays as $O\!\left(\sqrt{\tfrac{1}{n_e} \log \tfrac{1}{s}}\right)$: a per-estimate guarantee on the error of a single edge sensitivity at a given iteration using only discrete predictions rather than a uniform bound over the search or a convergence result for what remains a heuristic procedure. Because the stopping criterion in Eq.~\eqref{eq:boundary-stop} is verified directly against the victim, every recorded boundary pair holds regardless of estimate quality. Moreover, the search uses $\hat{g}_e^{(t)}$ only to rank candidates (Eq.~\eqref{eq:topk-edges}) so it needs correct ordering, not accurate magnitudes. Corollary~\ref{cor:selection} (Appendix~\ref{app:selection}) sharpens the bound above into a selection guarantee whose sample cost scales with the inverse squared sensitivity gap.

\subsection{Boundary Search via Estimated Edge Sensitivities}
\label{subsubsec:boundary-search}

Using the estimated sensitivities $\{\hat{g}_e^{(t)}\}_{e \in \mathcal{C}(G^{(t)})}$, we perform a discrete, gradient-guided search in graph space to locate decision boundaries. Since we aim to push the prediction away from the original label $y^{(0)}$, we select the $k$ edges with the largest positive estimated sensitivities:
\begin{equation}
    S^{(t)}
    \;=\;
    \operatorname{TopK}_{\max}\Bigl(\bigl\{\hat{g}_e^{(t)} : e \in \mathcal{C}(G^{(t)})\bigr\}, k\Bigr),
    \label{eq:topk-edges}
\end{equation}
where $S^{(t)} \subseteq \mathcal{C}(G^{(t)})$ contains the $k$ edges most likely to push the prediction across a decision boundary. We then update the graph by flipping all edges in $S^{(t)}$: $G^{(t+1)} = \operatorname{Update}(G^{(t)}, S^{(t)})$, where $\operatorname{Update}$ applies the flip operation in Eq.~\eqref{eq:edge-flip} to each $e \in S^{(t)}$. Starting from $G^{(0)} = G$, we iterate until reaching an iteration $t^\star$ such that
\begin{equation}
    y\bigl(G^{(t^\star)}\bigr) \neq y^{(0)}
    \quad \text{and} \quad
    \delta_e\bigl(G^{(0)}, G^{(t^\star)}\bigr) \le \delta.
    \label{eq:boundary-stop}
\end{equation}
At this point, we record the boundary pair $(G^{(0)}, G^{(t^\star)})$. The search (Algorithm~\ref{alg:boundary-search}) exits on exactly one of four conditions: \textbf{(i)} \textit{success}, when $y(G^{(t)}) \neq y^{(0)}$ within the edit budget $\delta$; \textbf{(ii)} \textit{edit-budget exhaustion}, when $\delta_e(G^{(0)}, G^{(t)}) \ge \delta$; \textbf{(iii)} \textit{stagnation}, when no candidate edge has positive estimated sensitivity ($S^{(t)} = \emptyset$); and \textbf{(iv)} the \textit{iteration cap} $t = T_{\max}$. Condition (iv) alone bounds the search at $T_G \le T_{\max} = 10$ iterations unconditionally, regardless of $\delta$, the explainer, or the victim, so a stopping iteration $t^\star$ always exists and the procedure always halts. It does not always halt successfully: exits (ii) and (iii) abandon the graph by design since Phase~2 is a per-graph budgeted attempt and the attacker proceeds to the next shadow graph.

\subsection{Explanation-Guided Query Reduction}
\label{subsubsec:explain-guided}

Naively, $\mathcal{C}(G^{(t)})$ could contain all possible edges, leading to $O(|V|^2)$ queries per iteration. To make the attack query-efficient, we leverage the explanatory subgraph produced by $h$ to restrict attention to a smaller, high-impact region of the graph.

Given $G^{(t)}$, we query the explainer to obtain
\[
G_{\mathrm{sub}}^{(t)} \;=\; G_{\mathrm{sub}}\bigl(G^{(t)}\bigr)
\;=\; \bigl(V_{\mathrm{sub}}^{(t)}, E_{\mathrm{sub}}^{(t)}, X_{\mathrm{sub}}^{(t)}\bigr).
\]
We define the one-hop neighborhood of the explanatory nodes as
\[
N\bigl(V_{\mathrm{sub}}^{(t)}\bigr)
\;=\;
\bigl\{ j \in V(G^{(t)}) : \exists\, i \in V_{\mathrm{sub}}^{(t)} \text{ with } (i,j) \in E(G^{(t)}) \bigr\}.
\]

We construct the candidate edge set as
{\small
\begin{equation}
\begin{split}
    \mathcal{C}\bigl(G^{(t)}\bigr)
    =&
    \underbrace{\bigl\{(i,j) : i,j \in V_{\mathrm{sub}}^{(t)},\, i < j\bigr\}}_{\text{within-explanation edges}} \cup\
    \underbrace{\bigl\{(i,j) : i \in V_{\mathrm{sub}}^{(t)},\, j \in N(V_{\mathrm{sub}}^{(t)}),\, i < j\bigr\}}_{\text{explanation-to-neighbor edges}}\cup
    \underbrace{\mathcal{S}_{\mathrm{carry}}^{(t)}}_{\text{carry-over edges}}
    \cup
    \underbrace{\mathcal{S}_{\mathrm{rand}}^{(t)}}_{\text{random exploration}}.
\end{split}
\label{eq:explain-candidates}
\end{equation}
}
Here $\mathcal{S}_{\mathrm{carry}}^{(t)}$ contains edges that were highly influential in previous iterations, and $\mathcal{S}_{\mathrm{rand}}^{(t)}$ contains randomly sampled edges for global exploration.

Since typical explanations satisfy $|V_{\mathrm{sub}}^{(t)}| \ll n$, the candidate set size is
\[
|\mathcal{C}(G^{(t)})|
\;=\; O\bigl(|V_{\mathrm{sub}}^{(t)}|^2
  + |V_{\mathrm{sub}}^{(t)}| \cdot d_{\mathrm{avg}}
  + |\mathcal{S}_{\mathrm{carry}}^{(t)}|
  + |\mathcal{S}_{\mathrm{rand}}^{(t)}|\bigr),
\]
where $d_{\mathrm{avg}}$ is the average degree. This reduces victim queries per iteration from $O(|V|^2 \cdot n_e)$ to $O(|\mathcal{C}(G^{(t)})| \cdot n_e)$, which largely increases our method's efficiency.

To illustrate concretely, consider a typical molecular graph with $n=18$ nodes and an explanation subgraph of $|V_{\mathrm{sub}}|=5$ nodes with average degree $d_{\mathrm{avg}}=3$. The candidate set contains $\binom{5}{2} + 5 \cdot 3 + 5 + 10 = 40$ edges versus $\binom{18}{2} = 153$ possible edges, a $3.8$ reduction. With $n_e=5$ MC samples, Phase~2 requires $\sim$400 queries per graph per iteration versus $\sim$1530 naively and in practice, most graphs resolve within $T \leq 3$ iterations.

\subsection{Summary: Method Overview}
\label{subsubsec:two-phase}
With all above components we design, we are able to summarize our whole method in the following overview. Let $\bar{\mathcal{G}}$ denote the set of boundary pairs obtained by applying the two-phase search procedure to all initial graphs:
\begin{equation}
    \bar{\mathcal{G}}
    \;=\;
    \bigl\{ (G, G') : y(G) \neq y(G'),\ \delta_e(G, G') \le \delta \bigr\}.
    \label{eq:boundary-set}
\end{equation}
We construct a training set for the surrogate model $g$:
\begin{equation}
    \mathcal{D}_{\mathrm{train}}
    \;=\;
    \bigl\{ (\bar{G}, y(\bar{G})) : \forall \bar{G} \in (G, G'),\ (G,G') \in \bar{\mathcal{G}} \bigr\}.
    \label{eq:train-set}
\end{equation}
The surrogate is trained via
\begin{equation}
    \min_{g \in \mathcal{H}}
    \sum_{(G,y) \in \mathcal{D}_{\mathrm{train}}}
    \ell\bigl(g(G), y\bigr),
    \label{eq:surrogate-training}
\end{equation}
where $\mathcal{H}$ denotes the chosen GNN hypothesis class. We employ a two-phase approach. \textbf{Phase~1} cheaply tests all single-edge removals ($O(|E(G)|)$ queries): if any removal changes the prediction we record the boundary pair, else the graph passes to Phase~2. Restricting to removals rather than also adding edges trades some yield for this lower cost, most where boundaries are scarce (Appendix~\ref{app:boundary-empirical}, Table~\ref{tab:toggle-deletion}); Phase~2 recovers the difference. \textbf{Phase~2} triggers only when Phase~1 yields insufficient pairs, applying the full MC boundary search (Sections~\ref{subsubsec:grad-estimation}--\ref{subsubsec:explain-guided}) to remaining graphs. This design reflects that \emph{predictions find easy boundaries; explanations unlock hard ones}. Phase~1 is quantity-limited on low boundary-rate datasets and quality-limited on feature-dominated datasets; our phase ablation (Appendix~\ref{sec:phase-ablation}) confirms both effects. The complete algorithm is provided in Appendix~\ref{app:algorithms} (Appendix Algorithm~\ref{alg:extraction} for the overall procedure, Appendix Algorithm~\ref{alg:boundary-search} for the \textsc{BoundarySearch} subroutine).

\section{Experiments}
We evaluate our method across eight graph classification benchmarks (Table~\ref{tab:datasets} in Appendix~\ref{app:datasets}), three victim architectures (GCN, GAT, GraphSAGE), and two explainer types (GNNExplainer, PGExplainer). We organize our evaluation around three research questions:

\begin{itemize}[leftmargin=*,nosep]
\item \textbf{RQ1}: Does boundary sampling improve extraction fidelity over random and high-confidence strategies?
\item \textbf{RQ2}: Does our method generalize across victim architectures and explainer types? (Appendix~\ref{app:gat-results},~\ref{app:sage-results})
\item \textbf{RQ3}: How do individual components (explanation guidance, MC estimation, two-phase design) contribute to performance? (Appendix~\ref{app:matched-noexp},~\ref{app:component-ablations},~\ref{sec:phase-ablation})
\end{itemize}

\subsection{Experimental Setup}

We report fidelity (surrogate-victim agreement on a class-balanced test set) and accuracy (surrogate vs.\ ground truth). All surrogates use the same architecture family as the victim with fixed hyperparameters to isolate the effect of training data quality. Full details are in Appendix~\ref{app:hyperparameters}.

\paragraph{Datasets.}
We evaluate on eight graph classification benchmarks spanning molecular (AIDS, MUTAG, NCI1, PTC\_FM, Tox21\_AhR), visual (MNIST, Letter-low), and synthetic (Synthie) domains. Dataset sizes range from 188 to 70{,}000 graphs, with five binary and three multi-class tasks (up to 15 classes), testing our method across diverse scales and decision boundary complexities. Full dataset statistics are provided in Table~\ref{tab:datasets} (Appendix~\ref{app:datasets}).

\paragraph{Evaluation Metrics.}
We report fidelity, defined as the agreement rate between surrogate and victim predictions on a class-balanced test set. Fidelity is our primary metric because it measures how faithfully the surrogate replicates the victim's decision logic, independent of the victim's own correctness. We additionally report accuracy (surrogate vs.\ ground-truth labels) as a secondary measure. Details on test set balancing are in Appendix~\ref{app:metrics}.

\paragraph{Baselines.} Existing GNN extraction attacks require either continuous victim outputs (probability scores, embeddings, or importance vectors)~\cite{shen2021modelstealingattacksinductive, podhajski2024efficientmodelstealingattacksinductive, ma2025explanationsleakdecisionlogic, 10.1145/3665451.3665533} or operate on node classification over a shared graph~\cite{wu2021modelextractionattacksgraph, 10.5555/3698900.3699194, GUAN2024112144}. Since no existing method operates under our discrete-only graph classification constraints, we construct baselines from straightforward strategies and adapt EGSteal~\cite{ma2025explanationsleakdecisionlogic}, which targets the same task but under a white-box setting with continuous outputs. Our adaptation retains only its style-invariance augmentation, as the core explanation alignment mechanism requires continuous importance vectors that are unavailable in our setting. Full details are in Appendix~\ref{app:baselines}:
\begin{itemize}[leftmargin=*,nosep]
\item \textbf{Shadow-Only}: Random shadow samples labeled by the victim, with no explanation access or boundary-seeking strategy.
\item \textbf{Non-boundary}: Shadow graphs where no single-edge removal changes the victim's prediction, representing interior-region samples far from decision boundaries.
\item \textbf{EGSteal-BB}: Adaptation of EGSteal's style-invariance augmentation to discrete explanation masks; perturbs edges outside the explanation subgraph while preserving the original label.
\end{itemize}
We additionally report a budget-matched \textbf{No-Exp} control that replaces explanation-guided candidate selection with uniform random selection at equal candidate count and query budget, isolating the explanation signal from raw computation (Appendix~\ref{app:matched-noexp}).

\paragraph{Our Methods.} We evaluate two variants of our boundary-based extraction framework:
\begin{itemize}[leftmargin=*,nosep]
\item \textbf{Boundary}: Full two-phase method (Section~\ref{subsubsec:two-phase}) combining cheap exhaustive $\delta_e=1$ search with explanation-guided MC boundary search.
\item \textbf{Hybrid}: Allocates 50\% of the query budget to boundary pair generation and 50\% to random sampling, combining targeted boundary information with broad distributional coverage.
\end{itemize}

\subsection{Main Results} 
\label{subsec:main-results}

\begin{table*}[t]
\centering
\resizebox{\textwidth}{!}{%
\begin{tabular}{@{}ll|cc|cc|cc|cc|cc@{}}
\toprule
& & \multicolumn{2}{c|}{\textbf{Shadow}} & \multicolumn{2}{c|}{\textbf{Non-boundary}} & \multicolumn{2}{c|}{\textbf{EGSteal-BB}} & \multicolumn{2}{c|}{\textbf{Boundary (XSTEAL)}} & \multicolumn{2}{c}{\textbf{Hybrid (XSTEAL)}} \\
\textbf{Dataset} & \textbf{Expl.} & Fid. & Acc. & Fid. & Acc. & Fid. & Acc. & Fid. & Acc. & Fid. & Acc. \\
\midrule
\multirow{2}{*}{AIDS}
  & GNNExp & \multirow{2}{*}{86.8$\pm$2.3} & \multirow{2}{*}{72.1$\pm$2.3} & \multirow{2}{*}{71.3$\pm$0.8} & \multirow{2}{*}{66.4$\pm$0.0} & 65.8$\pm$1.1 & 67.2$\pm$0.7 & \textbf{91.1$\pm$0.8} & 75.9$\pm$0.0 & 85.1$\pm$3.5 & 71.0$\pm$3.9 \\
  & PGExp  &  &  &  &  & 68.4$\pm$0.4 & 70.4$\pm$0.4 & 83.6$\pm$7.4 & 70.7$\pm$6.0 & 88.8$\pm$1.9 & 74.1$\pm$1.9 \\
\midrule
\multirow{2}{*}{Letter-low}
  & GNNExp & \multirow{2}{*}{68.6$\pm$2.1} & \multirow{2}{*}{69.6$\pm$1.8} & \multirow{2}{*}{44.4$\pm$1.6} & \multirow{2}{*}{43.0$\pm$1.0} & 68.9$\pm$1.6 & 68.4$\pm$1.9 & 63.0$\pm$3.7 & 65.7$\pm$2.7 & \textbf{72.3$\pm$0.7} & 72.3$\pm$0.7 \\
  & PGExp  &  &  &  &  & 66.4$\pm$3.9 & 66.9$\pm$5.5 & 69.9$\pm$3.7 & 71.6$\pm$4.9 & 69.4$\pm$1.8 & 69.1$\pm$1.7 \\
\midrule
\multirow{2}{*}{MNIST}
  & GNNExp & \multirow{2}{*}{53.4$\pm$1.3} & \multirow{2}{*}{26.4$\pm$0.8} & \multirow{2}{*}{48.3$\pm$2.5} & \multirow{2}{*}{24.4$\pm$0.4} & 53.8$\pm$1.7 & 26.5$\pm$0.6 & 54.9$\pm$2.3 & 26.2$\pm$0.6 & \textbf{55.1$\pm$1.7} & 26.6$\pm$0.4 \\
  & PGExp  &  &  &  &  & 50.2$\pm$0.8 & 25.1$\pm$0.3 & 52.0$\pm$2.6 & 25.7$\pm$0.8 & 53.1$\pm$3.1 & 25.9$\pm$0.7 \\
\midrule
\multirow{2}{*}{MUTAG}
  & GNNExp & \multirow{2}{*}{81.8$\pm$3.7} & \multirow{2}{*}{68.2$\pm$3.7} & \multirow{2}{*}{84.8$\pm$2.1} & \multirow{2}{*}{65.2$\pm$2.1} & 90.9$\pm$0.0 & 68.2$\pm$0.0 & \textbf{100.0$\pm$0.0} & 77.3$\pm$0.0 & 97.0$\pm$2.1 & 74.2$\pm$2.1 \\
  & PGExp  &  &  &  &  & 90.9$\pm$0.0 & 68.2$\pm$0.0 & \textbf{100.0$\pm$0.0} & 77.3$\pm$0.0 & 95.5$\pm$0.0 & 72.7$\pm$0.0 \\
\midrule
\multirow{2}{*}{PTC\_FM}
  & GNNExp & \multirow{2}{*}{77.4$\pm$2.1} & \multirow{2}{*}{54.8$\pm$2.1} & \multirow{2}{*}{75.0$\pm$0.0} & \multirow{2}{*}{50.0$\pm$0.0} & 59.5$\pm$2.1 & 44.1$\pm$2.1 & 85.7$\pm$3.6 & 53.6$\pm$3.6 & \textbf{90.5$\pm$2.1} & 58.3$\pm$2.1 \\
  & PGExp  &  &  &  &  & 59.5$\pm$2.1 & 44.1$\pm$2.1 & 84.5$\pm$2.1 & 52.4$\pm$2.1 & 82.1$\pm$9.5 & 50.0$\pm$9.4 \\
\midrule
\multirow{2}{*}{Synthie}
  & GNNExp & \multirow{2}{*}{39.7$\pm$5.5} & \multirow{2}{*}{35.3$\pm$5.0} & \multirow{2}{*}{47.4$\pm$5.5} & \multirow{2}{*}{44.2$\pm$4.2} & 42.9$\pm$4.0 & 37.8$\pm$3.3 & \textbf{51.3$\pm$5.5} & 44.2$\pm$4.2 & 46.8$\pm$3.9 & 40.4$\pm$6.8 \\
  & PGExp  &  &  &  &  & 44.9$\pm$5.0 & 37.8$\pm$3.3 & 46.2$\pm$3.1 & 43.6$\pm$7.4 & 46.2$\pm$0.0 & 40.4$\pm$4.2 \\
\midrule
\multirow{2}{*}{NCI1}
  & GNNExp & \multirow{2}{*}{94.3$\pm$1.7} & \multirow{2}{*}{66.3$\pm$0.2} & \multirow{2}{*}{91.4$\pm$0.3} & \multirow{2}{*}{65.7$\pm$0.3} & 88.9$\pm$1.0 & 65.4$\pm$0.3 & \textbf{96.0$\pm$0.2} & 66.3$\pm$0.2 & 92.2$\pm$2.8 & 65.9$\pm$0.6 \\
  & PGExp  &  &  &  &  & 84.8$\pm$1.3 & 63.5$\pm$0.7 & 95.2$\pm$0.5 & 66.5$\pm$0.2 & 92.6$\pm$2.6 & 65.9$\pm$0.8 \\
\midrule
\multirow{2}{*}{Tox21}
  & GNNExp & \multirow{2}{*}{60.7$\pm$5.1} & \multirow{2}{*}{54.8$\pm$6.1} & \multirow{2}{*}{54.8$\pm$3.4} & \multirow{2}{*}{48.8$\pm$3.4} & 54.8$\pm$3.4 & 48.8$\pm$3.4 & 89.3$\pm$2.9 & 83.3$\pm$4.5 & 77.4$\pm$9.4 & 76.2$\pm$7.3 \\
  & PGExp  &  &  &  &  & 57.1$\pm$5.8 & 51.2$\pm$3.4 & \textbf{92.9$\pm$2.9} & 82.1$\pm$2.9 & 91.7$\pm$1.7 & 80.9$\pm$1.7 \\
\bottomrule
\end{tabular}%
}
\caption{Extraction results on \textit{GCN} victim at matched training set sizes. Fid.\ = Fidelity (\%), Acc.\ = Accuracy (\%). Best fidelity per dataset in \textbf{bold}.}
\label{tab:main-gcn}
\end{table*}

In Table~\ref{tab:main-gcn} we present extraction results on GCN victims. While in Appendix, Figure~\ref{fig:bar-fidelity} visualizes fidelity across all datasets grouped by each explainer, and Figure~\ref{fig:fidelity-main} shows fidelity-vs-training-set-size trajectories on NCI1. Full results for GAT and GraphSAGE victims are in Appendix~\ref{app:gat-results}--\ref{app:sage-results}; trajectory plots across all architectures are in Appendix~\ref{app:trajectory}.

\begin{figure}[t]
    \centering
    \includegraphics[width=\textwidth]{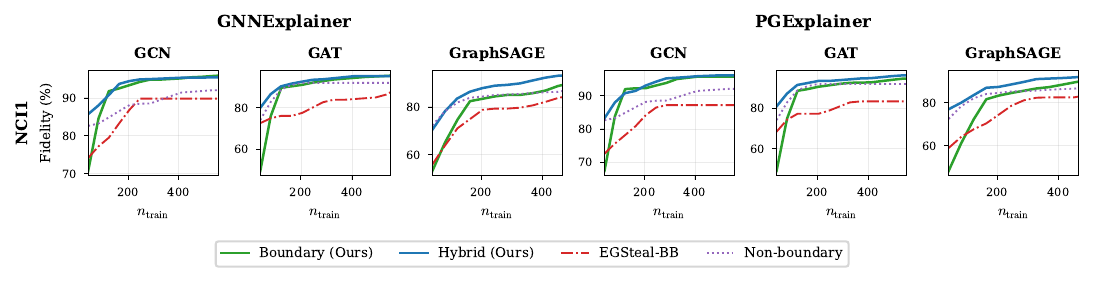}
    \caption{Fidelity (\%) across increasing training set size on NCI1.}
    \vspace{-1em}
    \label{fig:fidelity-main}
\end{figure}

\textbf{RQ1: Boundary sampling improves extraction fidelity.} On all datasets, Boundary or Hybrid achieves the highest fidelity on every dataset-explainer combination, with the largest gains where baselines struggle, such like Tox21 where Boundary (92.9\%) outperforms Shadow-Only (60.7\%) by 32.2 points, MUTAG (100.0\% vs.\ 81.8\%), and PTC\_FM (Hybrid 90.5\% vs.\ 77.4\%). We notice non-boundary consistently underperforms, confirming interior samples alone are insufficient. Additionally, Figure~\ref{fig:fidelity-main} shows Boundary and Hybrid separate from baselines early on NCI1 and converge to their advantage as training size grows. EGSteal-BB underperforms Shadow-Only on 5 of 8 datasets (GCN/GNNExplainer)---e.g., PTC\_FM: 59.5\% vs.\ 77.4\%. This notable insufficiency is consistent with its origin: EGSteal-BB is designed for white-box settings with continuous importance vectors, and our discrete adaptation necessarily strips its core alignment mechanism, leaving only style-invariance augmentation. This augmentation assumes the explanation mask causally partitions label-relevant from label-irrelevant structure; while under imprecise binary masks, this assumption breaks and the augmented samples introduce systematic label noise, which latter leads to additional fidelity loss in surrogate training (Appendix~\ref{app:egsteal-case-study}).

\textbf{RQ2: Generalization across architectures and explainers.} The outperforming pattern of our methods  holds across victim architectures: on GAT victims (Appendix Table~\ref{tab:main-gat}), Boundary achieves 97.6\% on Tox21 (GNNExplainer) and 95.5\% on MUTAG; on GraphSAGE (Appendix Table~\ref{tab:main-sage}), Boundary reaches 94.9\% on Tox21 and 93.3\% on PTC\_FM. Performance is consistent across both explainer types, with neither GNNExplainer nor PGExplainer systematically dominating. This robustness reflects a key design advantage: our method uses explanations \emph{heuristically} to narrow the candidate edge set rather than \emph{prescriptively} to assign labels. An imperfect explanation may degrade search \emph{efficiency} as we would explore a suboptimal region, but not \emph{correctness} since every boundary pair for surrogate training is verified by the victim's own predictions.

\begin{wrapfigure}{l}{0.5\textwidth}
    \centering
    \vspace{-10pt}
    \includegraphics[width=0.49\textwidth]{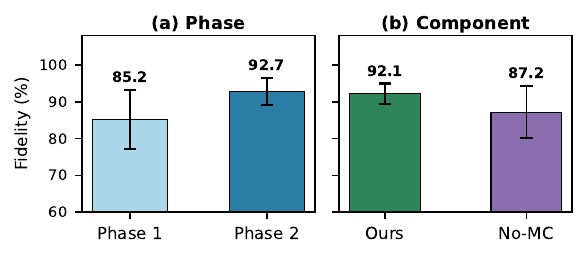}
    \vspace{-5pt}
    \caption{Phase and Component Ablations averaged across datasets and methods with error bars.}
    \label{fig:ablation-rq3}
    \vspace{-15pt}
\end{wrapfigure} 

\textbf{RQ3: Component contributions.} Figure~\ref{fig:ablation-rq3} summarizes both ablation studies. Phase~2 improves mean fidelity by 7.5 points over Phase~1 (92.7\% vs.\ 85.2\%) with substantially lower variance, confirming that explanation-guided MC search recovers boundary pairs that cheap exhaustive search misses---especially on low boundary-rate datasets. In the component study we evaluate No-Exp which draws candidate edges uniformly at random instead of the explanation subgraph and No-MC which uses a single sample to estimate each edge sensitivity. Under a budget-matched comparison (Appendix~\ref{app:matched-noexp}), explanation guidance improves fidelity over the No-Exp control on all three tested datasets, by $+10.7$, $+0.3$, and $+10.7$ points on PTC\_FM, NCI1, and Tox21\_AhR at matched training-set size with consistent gains at equal query budget. This confirms that the improvement comes from which candidates are selected rather than from additional computation. For the No-MC method, the component ablation (Figure~\ref{fig:ablation-rq3}b) shows that removing Monte Carlo sampling largely degrades its fidelity compared with ours (87.2\% vs 92.1\%). Moreover, we investigate No-MC's empirical details and found in most cases the boundary search is not able to find a boundary pair till exceeding the iteration limit. The lack of sufficient boundary samples thus further leads to the de-performance. This finding highlights the necessity and effectiveness of the Monte Carlo sampling.  Per-dataset breakdowns are in Appendix~\ref{app:component-ablations} and~\ref{sec:phase-ablation}.

\textbf{Multi-class performance.} On multi-class datasets (Letter-low, MNIST, Synthie), Hybrid consistently outperforms pure Boundary, as boundary search fragments across multiple decision surfaces and random coverage provides complementary distributional information. Absolute fidelity on MNIST ($55.1\%$) and Synthie ($51.3\%$) is notably lower than on binary tasks. However, these figures should be contextualized against the number of classes: with $10$ and $4$ classes respectively, uniform-chance fidelity is $10\%$ and $25\%$, making the achieved rates substantially above random agreement. While we notice that there is a performance gap between binary and multi-class settings, this observation is consistent with the fragmentation effect induced by our class-agnostic estimator (Section~\ref{subsubsec:grad-estimation}): in binary classification, all boundary pairs reinforce a single decision surface, while in multi-class settings the same query budget must be distributed across $\binom{C}{2}$ pairwise boundaries, diluting per-surface coverage. This analysis also suggests that scaling the query budget proportionally to the number of class boundaries is a promising direction for future work.

\section{Conclusion}
In this work we answer the question posed at the start of the paper affirmatively: subgraph explanations can be weaponized to steal graph neural networks. By recasting model extraction as a boundary pair search problem, we developed the first explanation-guided model extraction attack for GNN graph classification under strict black-box constraints. Our proposed Monte Carlo edge sensitivity estimation, backed by Hoeffding concentration guarantees, enables efficient boundary pair search by restricting estimated candidates to explanation-relevant subgraph regions. Extensive experiments across eight datasets, three victim architectures, and two explainer types demonstrate our attack's consistent superiority over adapted baselines, with fidelity gains of up to $32$ points on challenging datasets. Our baselines are necessarily adapted from methods built for other threat models since no prior attack runs under our constraints; we discuss this and further directions in Appendix~\ref{app:limitations}. These findings reveal a fundamental tension between regulatory demands for model transparency and the security of deployed GNN services, underscoring the need for defenses that can provide meaningful explanations without exposing decision boundary information.

\bibliography{arxiv}
\bibliographystyle{abbrvnat}

%%%%%%%%%%%%%%%%%%%%%%%%%%%%%%%%%%%%%%%%%%%%%%%%%%%%%%%%%%%%

\newpage
\appendix

\begin{center}
{\Large\textsc{Appendix Contents}}
\end{center}
\vspace{1em}
\startcontents[sections]
\printcontents[sections]{}{1}{\setcounter{tocdepth}{2}}

\newpage

\section{Related Work}
\paragraph{Model Extraction Attacks.} Model extraction attacks reconstruct target model functionality through query access. Early work demonstrated equation-solving attacks against simple models~\cite{10.5555/3241094.3241142}, later extended to deep neural networks~\cite{papernot2017practicalblackboxattacksmachine, 10.1145/3287560.3287562}. Recent approaches show that samples near decision boundaries maximize extraction efficiency~\cite{10.1145/3689932.3694756}. However, most methods assume access to probability scores or confidence values, assumptions violated in realistic deployments where APIs return only discrete labels. For GNNs specifically, existing attacks target node classification~\cite{10.5555/3698900.3699194, shen2021modelstealingattacksinductive}, exploiting shared graph topology and node-level embeddings. These methods fundamentally cannot transfer to graph classification, where each query involves an independent graph with only a graph-level prediction. Recent work has further explored both attack and defense strategies specific to GNN extraction~\cite{wang2025cegacosteffectiveapproachgraphbased,cheng2025atomframeworkdetectingquerybased}, though these also operate on node classification tasks. To our knowledge, no prior work addresses strict black-box model extraction for GNN graph classification with theoretical guarantees. Table~\ref{tab:positioning} positions representative extraction attacks and shows that only our setting operates on graph classification under discrete labels and discrete explanation masks. The closest methodological precedent is hard-label black-box evasion on graph classification~\cite{10.1145/3460120.3484796}. While it shares our hard-label query constraint and boundary-directed edge estimation, it solves a fundamentally different and simpler task: flipping one graph's own label; whereas our extraction task must reconstruct the victim's entire decision boundary by searching many boundary pairs to train a surrogate, further guided by explanations their setting does not use.

\paragraph{Explanation-Guided Attacks.} Regulatory requirements such as the EU AI Act~\cite{euai2024} increasingly mandate that MLaaS providers offer explanations, creating new attack surfaces. Methods like GNNExplainer~\cite{ying2019gnnexplainergeneratingexplanationsgraph}, PGExplainer~\cite{NEURIPS2020_e37b08dd}, and SubgraphX~\cite{yuan2021explainabilitygraphneuralnetworks} identify influential substructures for predictions. Recent work demonstrates that such explanations leak exploitable information to aid model extraction attacks~\cite{Oksuz_2024, ma2025explanationsleakdecisionlogic}. However, these approaches either require white-box access to explanation internals or assume continuous probability outputs. In realistic GMLaaS deployments, providers return only discrete explanation outputs (important nodes/edges) via APIs and not importance scores or gradients. Our work is the first to exploit explanations under strict black-box constraints where both predictions and explanations are discrete.

\paragraph{Defenses.} A natural mitigation is input-validity or out-of-distribution filtering: a provider that rejects queries whose structure falls off the data manifold blunts attacks that synthesize boundary graphs. This bears unevenly on our two phases: deletion-only Phase~1 removes existing edges and cannot introduce invalid substructure whereas Phase~2's toggle search can add edges that yield off-manifold graphs and is the more detectable stage. Such filtering trades coverage for security since legitimate rare inputs may also be rejected.

\begin{table}[H]
\centering
\resizebox{\textwidth}{!}{%
\begin{tabular}{@{}llll@{}}
\toprule
\textbf{Method} & \textbf{Task} & \textbf{Victim Output} & \textbf{Explanation Access} \\
\midrule
Node-clf GNN extraction~\cite{shen2021modelstealingattacksinductive, wu2021modelextractionattacksgraph, 10.5555/3698900.3699194} & Node clf. & Prob./embeddings & None \\
Explanation-guided extraction~\cite{Oksuz_2024, 10.1145/3665451.3665533, ma2025explanationsleakdecisionlogic} & Graph/node clf. & Continuous scores & Continuous (white-box) \\
Boundary-sample extraction~\cite{10.1145/3689932.3694756} & Non-graph & Probabilities & None \\
\textbf{XSTEAL (ours)} & Graph clf. & Discrete label & Discrete mask \\
\bottomrule
\end{tabular}%
}
\caption{Prior work by task, victim-output granularity, and explanation access. Only our setting operates on graph classification under discrete labels and discrete explanation masks.}
\label{tab:positioning}
\end{table}

\section{Limitations and Impact Statement}
\label{app:limitations}
While our work presents the first model extraction attack for GNN graph classification under strict black-box constraints, we acknowledge that no existing method operates under these same constraints. Therefore, due to the novelty of this work our baselines are necessarily adapted from methods designed for different threat models rather than direct competitors. We see several concrete directions this enables: native strict-black-box baselines for graph classification that do not require adapting white-box methods, extension to counterfactual and other explanation families, and joint structural-feature perturbations that generalize the attack beyond edge flips. Our findings demonstrate that subgraph explanations, increasingly mandated by regulations such as the EU AI Act, create exploitable attack surfaces for model extraction under strict black-box constraints. This work is intended to inform GMLaaS providers and regulators so that appropriate defenses can be developed.

\paragraph{Deployment Case Study: Toxicity Screening.} Regulatory transparency mandates increasingly reach high-risk classifiers: the EU AI Act~\cite{euai2024} and Colorado SB24-205~\cite{colorado_sb24205_2024} both require providers of such systems, including chemical toxicity screens, to return an explanation alongside each decision. Such a screen is a valuable target since it embeds proprietary assay data and substantial training cost, both of which a faithful surrogate lets a competitor bypass. To make this concrete, we instantiate our attack on Tox21\_AhR, where the adversary registers as an ordinary API user and holds a shadow set of public molecular graphs. On each query it observes only a discrete toxicity label and a binary explanation mask, with no access to gradients, logits, probabilities, or model parameters---precisely the interface the regulation obliges the provider to expose. Running the attack within a budget of 70\% of the shadow set (Appendices~\ref{app:data-splits} and~\ref{sec:query-complexity}), the resulting surrogate reaches 92.9\% fidelity to the victim, compared with 60.7\% for undirected shadow sampling at the same budget. The attack therefore succeeds using only the discrete outputs the mandate requires so the transparency requirement itself widens the very attack surface it is meant to make accountable.

\section{Boundary Sampling: Motivation and Theory}
\label{app:boundary-motivation}

\subsection{Empirical Validation}
\label{app:boundary-empirical}

\paragraph{Table Results.} Our results are outlined below, confirming the superiority of the boundary-sampling approach across all datasets.

\begin{table}[h]
\centering
\small
\begin{tabular}{@{}lcccc@{}}
\toprule
\textbf{Dataset} & \textbf{Boundary} & \textbf{Non-boundary} & \textbf{Hybrid} & \textbf{Shadow} \\
\midrule
AIDS & 86.2 & 71.6 & \textbf{91.4} & 88.8 \\
MUTAG & 95.5 & 81.8 & \textbf{100.0} & 86.4 \\
NCI1 & 94.1 & 91.1 & \textbf{95.3} & 95.0 \\
PTC\_FM & \textbf{82.1} & 71.4 & 71.4 & 57.1 \\
Tox21\_AhR & \textbf{89.3} & 50.0 & 75.0 & 50.0 \\
\bottomrule
\end{tabular}
\vspace{0.2em}
\caption{Fidelity (\%) on test sets with equal training set sizes. Best per dataset in \textbf{bold}.}
\label{tab:motivation-results}
\end{table}

\paragraph{Data Splits.} Each dataset is split 60\% victim training / 20\% shadow / 20\% test, stratified by class labels. The shadow set serves as the attacker's query pool; the test set is held out for fidelity evaluation.

\paragraph{Victim Models.} We train 3-layer GCN victim models via hyperparameter search over epochs $\in \{300, 500, 700, 1000\}$ and hidden dimensions $\in \{64, 128\}$, selecting the configuration with highest test accuracy. This reflects a realistic threat model: a model owner naturally deploys the best-performing model, and stronger victims represent a harder extraction target. Full configurations and test accuracies across all architectures are provided in Appendix Table~\ref{tab:victim_config_all}.

\paragraph{Boundary Detection.} For each graph $G$ in the shadow set, we exhaustively test all single-edge toggles ($\delta_e = 1$). Boundary detection statistics are reported in Table~\ref{tab:boundary-stats}.

\paragraph{Training Set Construction.} All four strategies (Boundary, Non-boundary, Hybrid, Shadow-Only) use equal-sized training sets, capped by the boundary pair count (the scarcest resource). Each boundary pair contributes two samples (one per class). Non-boundary and random sets are sampled to match. Hybrid uses 50\% boundary + 50\% random from the full shadow set.

\paragraph{Surrogate Training.} Surrogates use the same 3-layer GCN architecture (hidden dim 64). Hyperparameters: learning rate 0.001, no dropout, no weight decay, batch size 8, 500 epochs, Adam optimizer. Dropout is disabled because boundary pairs are near-identical graphs with opposite labels; dropout further obscures the subtle structural differences the model must learn.

\paragraph{Test Set Balancing.} Fidelity measures how well the surrogate replicates the victim's decision logic, not merely its majority-class tendency. We evaluate on a class-balanced subset of the test set (equal samples per victim-predicted class, fixed seed). Without balancing, a surrogate predicting only the majority class achieves artificially high fidelity. Balancing ensures fidelity reflects agreement on both sides of the decision boundary.

\paragraph{Boundary Detection Statistics.}
Table~\ref{tab:boundary-stats} reports single-edge boundary counts from the motivating experiment's toggling procedure ($\delta_e = 1$) across all datasets using GCN victims. Each boundary graph produces one pair (original + flipped) = 2 training samples. Low-rate datasets such as Tox21\_AhR rely heavily on Phase~2 for boundary pair generation.

\begin{table}[H]
\centering
\small
\begin{tabular}{@{}lcccc@{}}
\toprule
\textbf{Dataset} & \textbf{Shadow} & \textbf{Boundary} & \textbf{Non-Boundary} & \textbf{Rate} \\
\midrule
AIDS       & 400   & 127 & 273   & 31.8\% \\
MUTAG      & 38    & 11  & 27    & 28.9\% \\
NCI1       & 822   & 338 & 484   & 41.1\% \\
PTC\_FM    & 70    & 32  & 38    & 45.7\% \\
Tox21\_AhR & 1,634 & 69  & 1,565 & 4.2\%  \\
\bottomrule
\end{tabular}
\vspace{0.2em}
\caption{Single-edge boundary detection ($\delta_e = 1$, GCN victim). Counts are from the motivating experiment's \textit{toggling} procedure not Phase~1's deletion-only search.}
\label{tab:boundary-stats}
\end{table}

\paragraph{Toggle vs.\ Deletion Yield.} Table~\ref{tab:toggle-deletion} compares two single-edge detectors on the same shadow splits and GCN victims: toggling over all $O(|V|^2)$ node pairs (our motivating experiment) versus deletion over the $O(|E|)$ present edges (Phase~1). Insertions reach boundaries deletions cannot, and the gap widens where boundaries are scarce (Tox21\_AhR, $3.83\times$). Phase~1 trades this yield for its cheaper $O(|E|)$ cost; Phase~2 recovers the difference on low-yield datasets.
\begin{table}[H]
\centering
\small
\begin{tabular}{@{}lcccc@{}}
\toprule
\textbf{Dataset} & \textbf{Shadow graphs} & \textbf{Toggle (add/remove)} & \textbf{Deletion (remove only)} & \textbf{Ratio} \\
\midrule
AIDS       & 400     & 127 (31.8\%) & 80 (20.0\%)  & 1.59$\times$ \\
NCI1       & 822     & 338 (41.1\%) & 264 (32.1\%) & 1.28$\times$ \\
Tox21\_AhR & 1{,}634 & 69 (4.2\%)   & 18 (1.1\%)   & 3.83$\times$ \\
\bottomrule
\end{tabular}
\vspace{0.2em}
\caption{Single-edge boundary yield: toggling (add or remove) vs.\ deletion only, GCN victim. Toggle counts match Table~\ref{tab:boundary-stats}; insertions raise yield most where boundaries are scarce.}
\label{tab:toggle-deletion}
\end{table}

\subsection{Theoretical Analysis of Boundary Data Effectiveness }
\label{sec:boundary-proof}
We first record that the object our search seeks---a single-edge boundary pair---exists for an \textit{arbitrary} victim, with no continuity, smoothness, or Lipschitz assumption. The Lipschitz assumptions introduced afterward bound only how tightly such a pair constrains the surviving hypotheses, not whether it exists.

\begin{lemma}[Existence of single-edge boundary pairs]
\label{lem:boundary-existence}
Let $f:\mathcal{G}\to[C]$ be any victim classifier on graphs over a fixed node set $V$, with no assumption on $f$ beyond being a function. Suppose two graphs $G, G''$ on $V$ satisfy $f(G)\neq f(G'')$, and let $d = \|A(G)-A(G'')\|_0$ be their edge-edit distance. Then there exist a graph $H$ and a single edge $e$ with
\[
    \delta_e(G,H) < d,
    \qquad
    \delta_e(H, H\oplus e)=1,
    \qquad
    f(H)\neq f(H\oplus e),
\]
i.e., a single-edge boundary pair lying on the edit path between $G$ and $G''$.
\end{lemma}

\begin{proof}
Identifying each graph on $V$ with the upper-triangular part of its adjacency matrix, the graphs form the Boolean hypercube $\{0,1\}^{\binom{|V|}{2}}$ under single-edge toggles, in which any two vertices are joined by a path that toggles their differing coordinates one at a time. Toggling the $d$ edges on which $G$ and $G''$ differ yields a sequence $G=H_0, H_1,\dots,H_d=G''$ with $\delta_e(H_{i-1},H_i)=1$ for every $i$. The label sequence $f(H_0),\dots,f(H_d)$ starts at $f(G)$ and ends at $f(G'')\neq f(G)$, so some index $i$ satisfies $f(H_{i-1})\neq f(H_i)$. Setting $H=H_{i-1}$ and $e$ the edge toggled at step $i$ gives $f(H)\neq f(H\oplus e)$ with $\delta_e(G,H)=i-1<d$.
\end{proof}

Lemma~\ref{lem:boundary-existence} certifies that the boundary-pair search is never vacuous: whenever a shadow graph has any differently labeled graph within edit distance $d$, a single-edge (toggle) boundary pair exists inside that ball---exactly the object the motivating experiment and Phase~2 seek. The remaining results add an $L$-Lipschitz score assumption only to quantify how strongly a queried pair collapses the version space.

Let $\mathcal{H}$ be the hypothesis class of the extraction model mapping $\mathcal{X}\rightarrow\{-1,+1\}$. We assume the victim model $h^{*}\in \mathcal{H}$ and satisfies $L$-Lipschitz continuity. Given a set of graph data and corresponding query responses $\mathcal{G}_{T}=\{(G_{t},h^{*}(G_{t}))\}$, the victim space $\mathcal{V}_{T}\subseteq\mathcal{H}$ is the set of all hypotheses consistent with the predicted graph data:
$$\mathcal{V}_{T}=\{h\in\mathcal{H}|h(G_{t})=h^{*}(G_{t}), \forall t\in [1,\dots,T]\}$$

Besides, we define a pair of graph data $(G,G')$ as $\gamma$-boundary data if their distance satisfies $\delta_{e}+\delta_{f} =\gamma$ and they hold different queried predictions on the victim model, i.e., $h^{*}(G)\neq h^{*}(G')$. We also define the disagreement region $DIS(\mathcal{V}_{T})$ as the subset of the input graph space $\mathbb{G}$ where the surviving models in the victim space yield conflicting predictions:
$$DIS(\mathcal{V}_{T})=\{G\in\mathbb{G}|\exists h_{1},h_{2}\in\mathcal{V}_{T}:h_{1}(G)\neq h_{2}(G)\}$$
For each $h\in\mathcal H$, let $f_h:\mathbb G\to\mathbb R$ be its score
function and $h(G)=\operatorname{sign}(f_h(G))$. We assume every $f_h$ is
$L$-Lipschitz with respect to the graph distance
$d_{\mathbb G}(G,G')=\delta_e+\delta_f$.

\begin{lemma}[Boundary-pair margin constraint]
\label{lem:margin}
Let $(G,G')$ be a queried $\gamma$-boundary pair with
$d_{\mathbb G}(G,G')=\gamma$ and, without loss of generality,
$h^*(G)=+1$ and $h^*(G')=-1$. Suppose every hypothesis
$h\in\mathcal H$ is induced by an $L$-Lipschitz score function $f_h$.
Then any surviving hypothesis $h\in\mathcal V_{T+2}$ satisfies
\[
    0 < f_h(G) < L\gamma,
    \qquad
    -L\gamma < f_h(G') < 0 .
\]
Thus, for sufficiently small $\gamma$, every surviving hypothesis must place
its decision boundary inside an $O(\gamma)$ neighborhood between $G$ and $G'$.
\end{lemma}

\begin{proof}
Since $h\in\mathcal V_{T+2}$ is consistent with the newly queried labels,
we have $h(G)=h^*(G)=+1$ and $h(G')=h^*(G')=-1$. Hence
$f_h(G)>0$ and $f_h(G')<0$. By $L$-Lipschitz continuity,
\[
    |f_h(G)-f_h(G')|
    \leq
    L d_{\mathbb G}(G,G')
    =
    L\gamma .
\]
Because $f_h(G)>0$ and $f_h(G')<0$, we have
$f_h(G)-f_h(G')>f_h(G)$ and $f_h(G)-f_h(G')>-f_h(G')$.
Therefore,
\[
    0 < f_h(G) < L\gamma,
    \qquad
    -L\gamma < f_h(G') < 0 .
\]
\end{proof}
This lemma gives the local mechanism by which boundary pairs constrain the
version space: a tightly coupled pair with opposite victim labels forces every
surviving hypothesis to have scores close to the decision threshold at both
endpoints. Consequently, hypotheses whose decision boundary does not enter an
$O(\gamma)$ neighborhood of the queried pair become inconsistent with the
observed victim responses and are removed.

To give a further bound on extraction sample complexity, we define $B(h^{*},r)$ be the ball of hypotheses with an error rate of at most $r$,
\[
    B(h^*,r)
    =
    \left\{
        h\in\mathcal H:
        \mathbb P_{G\sim\mathcal D}
        \big(h(G)\neq h^*(G)\big)
        \leq r
    \right\}.
\]
and a disagreement coefficient $\theta_{h^{*}}(\epsilon)$ bounds the rate of version space collapse:
$$\theta_{h^{*}}(\epsilon)=\sup_{r\geq\epsilon}\frac{\mathbb{P}_{G\sim \mathcal{D}}(G\in DIS(B(h^{*},r)))}{r}$$

We assume that the boundary-pair sampler is non-degenerate in the following sense:
at each round, the queried $\gamma$-boundary pairs cover the current disagreement
region. More precisely, letting $A_T = DIS(\mathcal V_T)$, the marginal
distribution of queried boundary endpoints dominates the conditional
distribution $\mathcal D(\cdot \mid A_T)$ up to a universal constant. This
prevents the algorithm from repeatedly querying the same boundary pair and
ensures that boundary queries explore the informative part of the version
space.

\begin{theorem}[Sample Complexity of Boundary Extraction]
Assume realizability, $h^*\in\mathcal H$, and suppose the boundary-pair
sampler is non-degenerate: at epoch $k$, if
$A_k=DIS(\mathcal V_k)$, the endpoint distribution $Q_k$ of the queried
boundary pairs satisfies
\[
    Q_k(S)\geq c_0\,\mathcal D(S\mid A_k)
\]
for every measurable $S\subseteq A_k$ and for some constant $c_0>0$.
Then, with probability at least $1-\delta$, the number of queried endpoints
needed to obtain a hypothesis with generalization error at most $\epsilon$ is
\[
    N
    =
    O\!\left(
        \frac{\theta_{h^*}(\epsilon)}{c_0}
        \left(
            d_{VC}\log\frac{\theta_{h^*}(\epsilon)}{c_0}
            +
            \log\frac{\log(1/\epsilon)}{\delta}
        \right)
        \log\frac{1}{\epsilon}
    \right).
\]
In particular, when $c_0$ and logarithmic factors are suppressed,
\[
    N=\widetilde O\!\left(
        \theta_{h^*}(\epsilon)d_{VC}\log\frac{1}{\epsilon}
    \right).
\]
Counting boundary pairs instead of endpoints changes the bound only by a
constant factor.
\end{theorem}
\begin{proof}
Let
\[
    \operatorname{err}(h)
    =
    \mathbb P_{G \sim \mathcal D}\big(h(G) \neq h^*(G)\big),
    \qquad
    r_T
    =
    \max_{h \in \mathcal V_T} \operatorname{err}(h).
\]
Since the realizability assumption gives $h^* \in \mathcal V_T$, for any
$h \in \mathcal V_T$ and any graph $G$ such that $h(G) \neq h^*(G)$, the graph
$G$ must belong to the disagreement region $DIS(\mathcal V_T)$. Therefore,
\[
    \{G : h(G) \neq h^*(G)\}
    \subseteq
    DIS(\mathcal V_T).
\]
Moreover, by definition of $r_T$, we have
$\mathcal V_T \subseteq B(h^*, r_T)$. Hence
\[
    DIS(\mathcal V_T)
    \subseteq
    DIS(B(h^*, r_T)).
\]
By the definition of the disagreement coefficient, for any $r_T \geq \epsilon$,
\[
    \mathbb P_{G \sim \mathcal D}
    \big(G \in DIS(\mathcal V_T)\big)
    \leq
    \mathbb P_{G \sim \mathcal D}
    \big(G \in DIS(B(h^*, r_T))\big)
    \leq
    \theta_{h^*}(\epsilon) r_T .
\]
Denote $\theta = \theta_{h^*}(\epsilon)$ and $A_T = DIS(\mathcal V_T)$.
Consider any hypothesis $h \in \mathcal V_T$ whose error is larger than
$r_T/2$. Since its disagreement set with $h^*$ is contained in $A_T$, we have
\[
\begin{aligned}
    \mathbb P_{G \sim \mathcal D}
    \big(h(G) \neq h^*(G) \mid G \in A_T\big)
    &=
    \frac{
        \mathbb P_{G \sim \mathcal D}
        \big(h(G) \neq h^*(G)\big)
    }{
        \mathbb P_{G \sim \mathcal D}(G \in A_T)
    }  
    \geq
    \frac{r_T/2}{\theta r_T}
    =
    \frac{1}{2\theta}.
\end{aligned}
\]
Thus, conditioned on the current disagreement region, every hypothesis with
error larger than $r_T/2$ disagrees with the victim model on at least a
$1/(2\theta)$ fraction of the informative region.

Now let $Q_T$ be the marginal distribution of queried boundary endpoints at
round $T$. By the non-degeneracy assumption, for any measurable
$S\subseteq A_T$,
\[
    Q_T(S)\geq c_0\,\mathcal D(S\mid A_T).
\]
For any hypothesis $h\in\mathcal V_T$ with
$\operatorname{err}(h)>r_T/2$, define
\[
    S_h=\{G:h(G)\neq h^*(G)\}.
\]
Since $S_h\subseteq A_T$, the previous inequality gives
\[
    Q_T(S_h)
    \geq
    c_0\,\mathcal D(S_h\mid A_T)
    \geq
    \frac{c_0}{2\theta}.
\]
Therefore, under the endpoint query distribution $Q_T$, every hypothesis with
error larger than $r_T/2$ disagrees with the victim on a set of probability at
least $c_0/(2\theta)$.

Apply the VC $\epsilon$-net theorem to the class
\[
    \mathcal A
    =
    \big\{
        \{G:h(G)\neq h^*(G)\}:h\in\mathcal H
    \big\},
\]
whose VC dimension is $O(d_{VC})$. A batch of
\[
    m
    =
    O\!\left(
        \frac{\theta}{c_0}
        \left(
            d_{VC}\log\frac{\theta}{c_0}
            +
            \log\frac{1}{\delta_T}
        \right)
    \right)
\]
queried endpoints forms an $\epsilon$-net for all sets in $\mathcal A$ of
$Q_T$-measure at least $c_0/(2\theta)$ with probability at least
$1-\delta_T$. Hence every surviving hypothesis with error larger than
$r_T/2$ is hit by at least one queried endpoint and is eliminated. Therefore,
after this batch,
\[
    r_{T+m}\leq \frac{r_T}{2}.
\]

Starting from $r_0 \leq 1$, after
\[
    K = \left\lceil \log_2 \frac{1}{\epsilon} \right\rceil
\]
epochs, the version-space radius satisfies $r_K \leq \epsilon$. Therefore the
total number of queried boundary endpoints is
\[
\begin{aligned}
    N
    &=
    \sum_{k=1}^{K}
    O\!\left(
        \theta_{h^*}(\epsilon)
        \left(
            d_{VC}\log \theta_{h^*}(\epsilon)
            +
            \log \frac{K}{\delta}
        \right)
    \right)  \\
    &=
    O\!\left(
        \theta_{h^*}(\epsilon)
        \left(
            d_{VC}\log \theta_{h^*}(\epsilon)
            +
            \log \frac{\log(1/\epsilon)}{\delta}
        \right)
        \log \frac{1}{\epsilon}
    \right).
\end{aligned}
\]
Suppressing logarithmic confidence factors and the mild
$\log \theta_{h^*}(\epsilon)$ term gives
\[
    N
    =
    O\!\left(
        \theta_{h^*}
        \cdot d_{VC}
        \log \frac{1}{\epsilon}
    \right).
\]
Since each $\gamma$-boundary pair contributes two labeled endpoints, counting
queries by boundary pairs only changes the bound by a constant factor. This
completes the proof.
\end{proof}

\section{Algorithms}
\label{app:algorithms}

\subsection{Overall Extraction Procedure}
\label{app:overview-alg}

Algorithm~\ref{alg:extraction} presents the complete XSTEAL procedure, combining Phase~1 (exhaustive search) and Phase~2 (explanation-guided MC search).
\begin{algorithm}[H]
\small
\caption{XSTEAL: eXplanation-guided STEALing}
\label{alg:extraction}
\begin{algorithmic}[1]
\REQUIRE Victim $f$, explainer $h$, shadow set $\mathcal{G}_{\mathrm{shadow}}$, edge threshold $\delta$, query budget $Q$
\ENSURE Boundary dataset $\bar{\mathcal{G}}$, surrogate $g$
\STATE $\bar{\mathcal{G}} \leftarrow \emptyset$, \quad $\mathcal{R} \leftarrow \emptyset$ \hfill // $\mathcal{R}$: non-boundary graphs
\STATE \textbf{// Phase 1: exhaustive search}
\FORALL{$G \in \mathcal{G}_{\mathrm{shadow}}$}
    \FORALL{$e \in E(G)$}
        \STATE $G' \leftarrow G \ominus e$
        \IF{$y(G') \neq y(G)$}
            \STATE $\bar{\mathcal{G}} \leftarrow \bar{\mathcal{G}} \cup \{(G, G')\}$; \quad \textbf{break}
        \ENDIF
    \ENDFOR
    \IF{no boundary found}
        \STATE $\mathcal{R} \leftarrow \mathcal{R} \cup \{G\}$
    \ENDIF
\ENDFOR
\STATE \textbf{// Phase 2: explanation-guided MC search}
%\IF{$Q>0$}
    \FORALL{$G \in \mathcal{R}$ \textbf{while} $Q >0$}
        \STATE $(G',Q_{c}) \leftarrow \textsc{BoundarySearch}(G, f, h, \delta)$ \hfill // Alg.~\ref{alg:boundary-search}
        \STATE $Q=Q-Q_{c}$
        \IF{$G' \neq \texttt{null} \wedge Q\geq0$}
            \STATE $\bar{\mathcal{G}} \leftarrow \bar{\mathcal{G}} \cup \{(G, G')\}$
        \ENDIF
    \ENDFOR
%\ENDIF
\STATE Train surrogate $g$ on $\{(\bar{G}, f(\bar{G})) : \bar{G} \in \bar{\mathcal{G}\}}$
\STATE \textbf{return} $\bar{\mathcal{G}}, g$
\end{algorithmic}
\end{algorithm}

\subsection{Boundary Search Subroutine}
\label{app:algorithm}

Algorithm~\ref{alg:boundary-search} provides the complete \textsc{BoundarySearch} subroutine referenced in Algorithm~\ref{alg:extraction}.

\begin{algorithm}[]
\small
\caption{\textsc{BoundarySearch}: Find boundary pair for a single graph}
\label{alg:boundary-search}
\begin{algorithmic}[1]
\REQUIRE Graph $G$, victim $f$, explainer $h$, edge threshold $\Delta_e$, max iterations $T_{\max}$, edges per step $k$, MC samples $n_e$, perturbation probability $p_{\mathrm{flip}}$, carry-over size $|\mathcal{S}_{\mathrm{carry}}|$, random exploration size $|\mathcal{S}_{\mathrm{rand}}|$
\ENSURE Boundary graph $G'$ and Query Cost $Q_{c}$
\STATE $G^{(0)} \leftarrow G$, \quad $y^{(0)} \leftarrow f(G)$, \quad $t \leftarrow 0$, \quad $\mathcal{S}_{\mathrm{carry}} \leftarrow \emptyset$
\WHILE{$t < T_{\max}$}
    \STATE $y^{(t)} \leftarrow f(G^{(t)})$
    \STATE $Q_{c} \leftarrow Q_{c}+1$
    \STATE $\delta_e \leftarrow \|A(G^{(0)}) - A(G^{(t)})\|_0$
    \IF{$y^{(t)} \neq y^{(0)}$ \textbf{and} $\delta_e \leq \Delta_e$}
        \STATE \textbf{return} $(G^{(t)},Q_{c})$
    \ENDIF
    \IF{$\delta_e \geq \Delta_e$}
        \STATE \textbf{break}
    \ENDIF
    \STATE $G_{\mathrm{sub}}^{(t)} \leftarrow h(G^{(t)}, y^{(t)})$ \hfill // query explainer
    \STATE Construct candidate edges $\mathcal{C}(G^{(t)})$ via Eq.~\eqref{eq:explain-candidates}
    \FORALL{$e \in \mathcal{C}(G^{(t)})$}
        \STATE Estimate $\hat{g}_e^{(t)}$ via Eq.~\eqref{eq:empirical-sensitivity} with $n_e$ samples
        \STATE $Q_{c}\leftarrow Q_{c}+n_{e}$
    \ENDFOR
    \STATE $S^{(t)} \leftarrow \operatorname{TopK}_{\max}(\{\hat{g}_e^{(t)} : \hat{g}_e^{(t)} > 0\}, k)$
    \IF{$S^{(t)} = \emptyset$}
        \STATE \textbf{break}
    \ENDIF
    \STATE $G^{(t+1)} \leftarrow$ flip edges in $S^{(t)}$ \hfill // apply Eq.~\eqref{eq:edge-flip} to each $e \in S^{(t)}$
    \STATE $\mathcal{S}_{\mathrm{carry}} \leftarrow$ top-$|\mathcal{S}_{\mathrm{carry}}|$ edges from $S^{(t)}$ by sensitivity
    \STATE $t \leftarrow t + 1$
\ENDWHILE
\STATE \textbf{return} (\texttt{null},$Q_{c}$)
\end{algorithmic}
\end{algorithm}

%% ============================================================
%% C: Experimental Details
%% ============================================================
\section{Experimental Details}
\label{app:hyperparameters}
We provide the details for our experiments here. The implementation code is available at: \url{https://github.com/LabRAI/XSTEAL/}.

\subsection{Datasets}
\label{app:datasets}

We evaluate on graph classification benchmarks from the TUDataset collection~\cite{morris2020tudatasetcollectionbenchmarkdatasets} and the GNN Benchmark suite~\cite{dwivedi2022benchmarkinggraphneuralnetworks}, with several datasets originating from the IAM Graph Database~\cite{10.1007/978-3-540-89689-0_33}. Table~\ref{tab:datasets} summarizes dataset statistics.

\begin{table}[h]
\centering
\small
\begin{tabular}{@{}lccccc@{}}
\toprule
\textbf{Dataset} & \textbf{\#Graphs} & \textbf{\#Classes} & \textbf{Avg. $|V|$} & \textbf{Avg. $|E|$} & \textbf{Features} \\
\midrule
AIDS & 2,000 & 2 & 15.7 & 16.2 & 38 \\
Letter-low & 2,250 & 15 & 4.7 & 3.1 & 2 \\
MNIST & 70,000 & 10 & 70.6 & 282.3 & 3 \\
MUTAG & 188 & 2 & 17.9 & 19.8 & 7 \\
PTC\_FM & 349 & 2 & 14.1 & 14.5 & 14 \\
Synthie & 400 & 4 & 95.0 & 172.9 & 15 \\
NCI1 & 4,110 & 2 & 29.9 & 32.3 & 37 \\
Tox21\_AhR & 8,169 & 2 & 18.1 & 18.5 & 50 \\
\bottomrule
\end{tabular}
\vspace{0.2em}
\caption{Dataset statistics.}
\label{tab:datasets}
\end{table}

\paragraph{Dataset Selection Rationale.}
Our benchmark suite spans three key dimensions of diversity:
\begin{itemize}[leftmargin=*,nosep]
    \item \textbf{Domain diversity:} AIDS, MUTAG, NCI1, and PTC\_FM represent molecular chemistry (drug discovery and toxicity screening), Tox21\_AhR encodes toxicological assay activity (regulatory safety testing), MNIST captures handwritten digit recognition via superpixel graphs~\cite{monti2016geometricdeeplearninggraphs} (computer vision), Letter-low captures handwritten letter recognition from the IAM Graph Database~\cite{10.1007/978-3-540-89689-0_33} (computer vision), and Synthie provides a synthetic benchmark with controlled properties and rich node attributes.
    \item \textbf{Classification complexity:} We include five binary classification datasets (AIDS, MUTAG, NCI1, PTC\_FM, Tox21\_AhR) where a single decision boundary allows all boundary pairs to reinforce the same decision surface, and three multi-class datasets (Synthie with 4 classes, MNIST with 10 classes, and Letter-low with 15 classes) to evaluate scalability of boundary search across increasingly complex decision boundary geometries.
    \item \textbf{Scale diversity:} Dataset sizes range from 188 graphs (MUTAG) to 70,000 graphs (MNIST), testing our method under both data-scarce and data-rich regimes. Graph sizes range from very small (Letter-low: $\sim$5 nodes, where single edge flips constitute $\sim$30\% structural change) to large dense graphs (MNIST: $\sim$71 nodes, $\sim$282 edges) and large sparse graphs (Synthie: $\sim$95 nodes), testing our method's efficiency across scales. Feature representations vary from rich atom-level attributes (AIDS: 38 features, Tox21: 50 features) to minimal spatial coordinates (Letter-low: 2 features) to superpixel descriptors (MNIST: 3 features), testing robustness to feature informativeness.
\end{itemize}

\subsection{Evaluation Metrics}
\label{app:metrics}

We report two primary metrics:
\begin{itemize}[leftmargin=*,nosep]
    \item \textbf{Fidelity}: Agreement rate between surrogate and victim
    predictions on a class-balanced test set:
    $\text{Fidelity} = \frac{1}{|\mathcal{D}_{\text{test}}|}
    \sum_{G \in \mathcal{D}_{\text{test}}} \mathbf{1}[g(G) = f(G)]$.
    \item \textbf{Accuracy}: Classification accuracy of the surrogate on
    ground-truth test labels.
\end{itemize}
Fidelity is our primary metric as it measures how faithfully the
surrogate replicates the victim's decision logic, independent of the
victim's own correctness.

\subsection{Victim Model Configuration}
\label{app:victim-config}

We employ three representative GNN architectures as victim models: Graph Convolutional Network (GCN), Graph Attention Network (GAT), and GraphSAGE. Each architecture consists of three message-passing layers followed by a global mean pooling operation and a linear classifier. We use ReLU activations and apply dropout ($p=0.5$) between layers for regularization. For each architecture-dataset combination, we perform a hyperparameter search over hidden dimensions $\in \{64, 128\}$ and training epochs $\in \{300, 500, 700, 1000\}$, selecting the configuration achieving highest test accuracy. All models are trained using the Adam optimizer with learning rate 0.001, weight decay 5e-4, and batch size 32. Table~\ref{tab:victim_config_all} summarizes the selected configurations and test accuracies across all architectures and datasets.

\begin{table*}[t]
\centering
\small
\begin{tabular}{@{}l|l|ccc@{}}
\toprule
\textbf{Architecture} & \textbf{Dataset} & \textbf{Hidden Dim} & \textbf{Selected Epochs} & \textbf{Test Accuracy (\%)} \\
\midrule
\multirow{8}{*}{\textbf{GCN}}
& AIDS              & 128 & 1000 & 87.0 \\
& Letter-low        & 128 & 1000 & 86.2 \\
& MNIST             & 128 & 200  & 38.1 \\
& MUTAG             & 64  & 300  & 78.9 \\
& PTC\_FM           & 64  & 500  & 65.7 \\
& Synthie           & 128 & 700  & 65.0 \\
& NCI1              & 64  & 1000 & 67.3 \\
& Tox21\_AhR        & 64  & 700  & 89.0 \\
\midrule
\multirow{8}{*}{\textbf{GAT}}
& AIDS              & 128 & 1000 & 88.5 \\
& Letter-low        & 64  & 700  & 98.0 \\
& MNIST             & 128 & 500  & 59.6 \\
& MUTAG             & 64  & 300  & 78.9 \\
& PTC\_FM           & 64  & 300  & 67.1 \\
& Synthie           & 128 & 300  & 75.0 \\
& NCI1              & 64  & 1000 & 69.6 \\
& Tox21\_AhR        & 64  & 700  & 89.5 \\
\midrule
\multirow{8}{*}{\textbf{GraphSAGE}}
& AIDS              & 128 & 500  & 90.5 \\
& Letter-low        & 128 & 700  & 99.3 \\
& MNIST             & 128 & 500  & 54.7 \\
& MUTAG             & 64  & 300  & 76.3 \\
& PTC\_FM           & 64  & 700  & 64.3 \\
& Synthie           & 128 & 1000 & 93.8 \\
& NCI1              & 128 & 1000 & 73.6 \\
& Tox21\_AhR        & 64  & 300  & 89.5 \\
\bottomrule
\end{tabular}
\caption{Victim model configurations. Best hidden dimension and epochs selected per dataset via grid search.}
\label{tab:victim_config_all}
\end{table*}

\subsection{Explainer Configuration}
\label{app:explainer-config}

PGExplainer is trained on the victim model for 100 epochs with learning rate 0.003. Explanations are thresholded at the median importance score to produce binary edge masks. GNNExplainer uses 100 optimization steps per instance with node masks (\texttt{node\_mask\_type='object'}). Both explainer types operate on a \texttt{deepcopy} of the victim model to prevent gradient hooks from corrupting victim predictions during the experiment (see Appendix~\ref{app:impl-notes}).

\subsection{Attack Configuration}
\label{app:attack-config}

Default hyperparameters for Phase~2 (explanation-guided MC search): $T_{\max} = 10$ (maximum iterations), $\delta_{\max} = 20$ (maximum edge edits), $k = 3$ (edges flipped per step), $n_e = 5$ (Monte Carlo samples per candidate edge), $p_{\mathrm{flip}} = 0.05$ (perturbation probability), $|\mathcal{S}_{\mathrm{carry}}| = 5$ (carry-over edges from previous iteration), $|\mathcal{S}_{\mathrm{rand}}| = 10$ (random exploration edges per iteration). Explanation masks are thresholded at the median importance score to select important nodes/edges.

\subsection{Surrogate Training}
\label{app:surrogate-training}

Surrogates use the same architecture family as the victim (GCN, GAT, or GraphSAGE) with fixed hyperparameters across all methods: hidden dimension 64, 3 layers, no dropout, learning rate 0.001, 500 epochs, batch size 8, Adam optimizer. Hyperparameters are deliberately held constant to isolate the effect of training data quality---the only variable across methods is how training samples are selected and generated. This reflects realistic black-box constraints where the attacker does not know the victim's exact configuration.

\subsection{Data Splits and Query Budget}
\label{app:data-splits}

Each dataset is split 60\% victim training / 20\% shadow / 20\% test, stratified by class labels. The shadow set serves as the attacker's query pool. For main results, we set the query budget to 70\% of the shadow set and match all methods to the same training set size, determined by boundary pair availability (the scarcest resource). Table~\ref{tab:ntrain} reports the matched $n_{\text{train}}$ for each dataset and victim architecture. This ensures fair comparison by isolating the effect of training data quality rather than quantity.

Boundary's actual training set size ($n_{\text{train}}$) may fall below the query budget when the shadow set contains insufficient boundary pairs. To ensure fair comparison, baseline methods are evaluated at the budget level yielding the closest $n_{\text{train}}$ to Boundary's output.

\subsection{Compute Resources}
\label{app:compute}
All experiments were conducted on Google Colab Pro+ using a single NVIDIA T4 GPU with high-RAM runtime. The full experimental pipeline (including victim training, boundary search, surrogate training, and evaluation across all datasets, architectures, and explainers) required approximately 600 Colab compute units.

\begin{table}[H]
\centering
\small
\begin{tabular}{@{}lccc@{}}
\toprule
\textbf{Dataset} & \textbf{GCN} & \textbf{GAT} & \textbf{GraphSAGE} \\
\midrule
AIDS       & 280 & 280 & 280 \\
Letter-low & 122 & 82  & 116 \\
MNIST      & 272 & 88  & 84  \\
MUTAG      & 26  & 26  & 26  \\
PTC\_FM    & 48  & 48  & 48  \\
Synthie    & 16  & 24  & 8   \\
NCI1       & 574 & 574 & 575 \\
Tox21\_AhR & 384 & 1,142 & 768 \\
\bottomrule
\end{tabular}
\vspace{0.5em}
\caption{Matched training set size ($n_{\text{train}}$) per dataset and victim architecture.}
\label{tab:ntrain}
\end{table}

\subsection{Implementation Notes}
\label{app:impl-notes}

\paragraph{Data Generation.}
Boundary and Hybrid are \emph{generative} methods: the flipped graph in each boundary pair is a synthetic sample created by edge perturbation, not an existing shadow graph. In contrast, Shadow-Only, Non-boundary, and EGSteal-BB only relabel existing graphs from the shadow set. This data augmentation around decision boundaries is a key mechanism behind the fidelity improvement, and enables our methods to generate informative training data even on very limited shadow datasets.

\paragraph{Explainer Model Isolation.}
PyTorch Geometric's explainer frameworks (GNNExplainer, PGExplainer) register forward hooks on the model that modify internal state in-place. Running an explainer on the victim model can corrupt its predictions for subsequent queries. We resolve this by creating a \texttt{copy.deepcopy} of the victim model for each explainer instantiation, ensuring the original victim remains unmodified throughout the experiment.

\paragraph{Test Set Balancing.}
Fidelity measures how well the surrogate replicates the victim's decision \emph{logic}---its learned mapping from graph structure to class label---not merely its majority-class tendency. We evaluate on a class-balanced subset of the test set. Without balancing, a surrogate predicting only the majority class achieves artificially high fidelity. Balancing ensures fidelity reflects agreement on both sides of the decision boundary, with single-class prediction scoring exactly 50\%.

\subsection{Baseline Details}
\label{app:baselines}

\paragraph{Shadow-Only.} Samples graphs from the shadow dataset and queries the victim for labels, with no explanation access or boundary-seeking strategy. This represents the standard model extraction baseline where training data is drawn from the attacker's available distribution without any targeted selection.

\paragraph{Non-boundary.} Trained exclusively on shadow graphs where no single-edge removal changes the victim's prediction (identified during Phase~1). These graphs lie in the interior of decision regions, far from boundaries. This baseline tests whether interior samples alone suffice for extraction, and serves as a direct contrast to boundary-based methods.

\paragraph{EGSteal-BB.} We adapt EGSteal~\cite{ma2025explanationsleakdecisionlogic} to our discrete setting. The original method requires continuous importance vectors $E \in \mathbb{R}^{|V|}$ for rank-based explanation alignment, which is infeasible under binary masks. We retain only its style-invariance augmentation: perturbing edges outside $G_{\text{sub}}$ while assigning the original victim label, generating 2 augmented samples per graph. This assumes $G_{\text{sub}}$ causally determines the prediction, so style perturbations are label-preserving. When explainers produce imprecise masks, this assumption breaks down and the augmented samples introduce label noise, which can degrade fidelity below naive random sampling (see case study in Appendix~\ref{app:egsteal-case-study}).

\paragraph{Boundary (XSTEAL).}
Our full two-phase method (Section~\ref{subsubsec:two-phase}). Phase~1 performs cheap exhaustive $\delta_e = 1$ edge removal search; Phase~2 triggers explanation-guided MC search when Phase~1 produces insufficient pairs. Both original and flipped graphs are added to the training set with their respective victim predictions. Boundary is a \emph{generative} method: the flipped graph in each pair is a synthetic sample created by edge perturbation, not an existing shadow graph.

\paragraph{Hybrid (XSTEAL).}
Allocates 50\% of the query budget to boundary pair generation and 50\% to random sampling. This combines the complementary strengths of both approaches: boundary pairs provide targeted information about the victim's decision logic near class transitions, while random samples provide broad coverage of the input distribution. Hybrid mitigates the sample limitation of pure Boundary on datasets with low boundary rates.

%% ============================================================
%% D: Additional Results
%% ============================================================
\section{Additional Results}
\label{app:additional}

We conduct ablations across three dimensions to assess robustness and isolate component contributions.

\paragraph{Victim Architecture Ablation.} We evaluate our method across all three victim architectures to assess generalization:
\begin{itemize}[leftmargin=*,nosep]
    \item GCN~\cite{kipf2017semisupervisedclassificationgraphconvolutional}: Graph Convolutional Networks (main text, Table~\ref{tab:main-gcn}).
    \item GAT~\cite{veličković2018graphattentionnetworks}: Graph Attention
    Networks with learned attention weights
    (Section~\ref{app:gat-results}).
    \item GraphSAGE~\cite{hamilton2018inductiverepresentationlearninglarge}:
    Sampling and aggregation-based architecture
    (Section~\ref{app:sage-results}).
\end{itemize}

\paragraph{Explainer Ablation.} We evaluate robustness to the explanation method deployed by the model owner by comparing two representative explainers:
\begin{itemize}[leftmargin=*,nosep]
    \item \textbf{GNNExplainer}~\cite{ying2019gnnexplainergeneratingexplanationsgraph}:
    Instance-level optimization-based explanations using node masks.
    \item \textbf{PGExplainer}~\cite{NEURIPS2020_e37b08dd}: Parameterized
    explainer trained across graphs using edge masks.
\end{itemize}
This ablation tests whether our method's gains are tied to a specific explanation format or generalize across explainer types. Performance differences between explainers reflect variation in the informativeness of leaked explanation structure rather than a limitation of our method.

\paragraph{Component Ablations.} We evaluate two ablations isolating key contributions of our method:
\begin{itemize}[leftmargin=*,nosep]
    \item \textbf{No-Exp}: Boundary search with random edge candidate selection instead of explanation-guided selection (Eq.~\ref{eq:explain-candidates}). Candidate edges are sampled uniformly at random from all $\binom{|V|}{2}$ possible node pairs, with the candidate-set size matched to the explanation-guided set per graph and per iteration, and both methods run at the default attack hyperparameters ($\delta_{\max} = 20$, $T_{\max} = 10$, $n_e = 5$, $k = 3$) so that only the candidate selection differs. This isolates the contribution of the explainer; full results are in Appendix~\ref{app:matched-noexp}.
    \item \textbf{No-MC}: Boundary search with deterministic
    single-sample estimation ($n_e = 1$, $p_{\text{flip}} = 0$) instead of
    Monte Carlo sampling ($n_e = 5$, $p_{\text{flip}} = 0.05$). This
    isolates the contribution of stochastic estimation.
\end{itemize}

\subsection{GAT Victim Results}
\label{app:gat-results}

\begin{table}[H]
\centering
\resizebox{\textwidth}{!}{%
\begin{tabular}{@{}ll|cc|cc|cc|cc|cc@{}}
\toprule
& & \multicolumn{2}{c|}{\textbf{Shadow}} & \multicolumn{2}{c|}{\textbf{Non-boundary}} & \multicolumn{2}{c|}{\textbf{EGSteal-BB}} & \multicolumn{2}{c|}{\textbf{Boundary (XSTEAL)}} & \multicolumn{2}{c}{\textbf{Hybrid (XSTEAL)}} \\
\textbf{Dataset} & \textbf{Expl.} & Fid. & Acc. & Fid. & Acc. & Fid. & Acc. & Fid. & Acc. & Fid. & Acc. \\
\midrule
\multirow{2}{*}{AIDS}
  & GNNExp & \multirow{2}{*}{86.9} & \multirow{2}{*}{76.1} & \multirow{2}{*}{68.2} & \multirow{2}{*}{72.2} & 67.0 & 61.9 & \textbf{89.8} & 75.6 & 86.9 & 72.7 \\
  & PGExp  &  &  &  &  & 64.2 & 62.5 & 88.6 & 77.8 & 85.8 & 75.0 \\
\midrule
\multirow{2}{*}{Letter-low}
  & GNNExp & \multirow{2}{*}{78.3} & \multirow{2}{*}{77.6} & \multirow{2}{*}{71.7} & \multirow{2}{*}{71.2} & 82.9 & 82.6 & 73.6 & 72.9 & 81.7 & 81.7 \\
  & PGExp  &  &  &  &  & 84.0 & 83.8 & 81.7 & 81.4 & \textbf{86.0} & 86.4 \\
\midrule
\multirow{2}{*}{MNIST}
  & GNNExp & \multirow{2}{*}{33.6} & \multirow{2}{*}{27.6} & \multirow{2}{*}{36.1} & \multirow{2}{*}{29.4} & 35.9 & 28.9 & 28.9 & 24.8 & 33.4 & 27.4 \\
  & PGExp  &  &  &  &  & 31.6 & 26.4 & 28.5 & 24.5 & \textbf{37.2} & 30.7 \\
\midrule
\multirow{2}{*}{MUTAG}
  & GNNExp & \multirow{2}{*}{81.8} & \multirow{2}{*}{68.2} & \multirow{2}{*}{86.4} & \multirow{2}{*}{63.6} & 86.4 & 72.7 & \textbf{95.5} & 72.7 & \textbf{95.5} & 72.7 \\
  & PGExp  &  &  &  &  & 90.9 & 68.2 & \textbf{95.5} & 72.7 & \textbf{95.5} & 72.7 \\
\midrule
\multirow{2}{*}{PTC\_FM}
  & GNNExp & \multirow{2}{*}{85.7} & \multirow{2}{*}{57.1} & \multirow{2}{*}{64.3} & \multirow{2}{*}{50.0} & 64.3 & 35.7 & 71.4 & 42.9 & \textbf{92.9} & 64.3 \\
  & PGExp  &  &  &  &  & 57.1 & 28.6 & 64.3 & 35.7 & 85.7 & 57.1 \\
\midrule
\multirow{2}{*}{Synthie}
  & GNNExp & \multirow{2}{*}{38.2} & \multirow{2}{*}{44.1} & \multirow{2}{*}{36.8} & \multirow{2}{*}{30.9} & 30.9 & 27.9 & 38.2 & 44.1 & 29.4 & 32.4 \\
  & PGExp  &  &  &  &  & 38.2 & 32.4 & 32.4 & 38.2 & \textbf{39.7} & 26.5 \\
\midrule
\multirow{2}{*}{NCI1}
  & GNNExp & \multirow{2}{*}{95.0} & \multirow{2}{*}{68.0} & \multirow{2}{*}{90.8} & \multirow{2}{*}{67.7} & 87.2 & 66.5 & 93.2 & 68.3 & \textbf{95.5} & 69.0 \\
  & PGExp  &  &  &  &  & 79.5 & 64.0 & \textbf{95.5} & 69.0 & 93.5 & 68.8 \\
\midrule
\multirow{2}{*}{Tox21}
  & GNNExp & \multirow{2}{*}{91.3} & \multirow{2}{*}{80.2} & \multirow{2}{*}{50.8} & \multirow{2}{*}{66.7} & 80.2 & 70.6 & \textbf{97.6} & 80.2 & 93.7 & 77.8 \\
  & PGExp  &  &  &  &  & 55.6 & 63.5 & 93.7 & 77.8 & 92.9 & 80.2 \\
\bottomrule
\end{tabular}%
}
\caption{Extraction results on \textit{GAT} victim at matched training set sizes. Fid.\ = Fidelity (\%), Acc.\ = Accuracy (\%). Best fidelity per dataset in \textbf{bold}.}
\label{tab:main-gat}
\end{table}

\subsection{GraphSAGE Victim Results}
\label{app:sage-results}

\begin{table}[H]
\centering
\resizebox{\textwidth}{!}{%
\begin{tabular}{@{}ll|cc|cc|cc|cc|cc@{}}
\toprule
& & \multicolumn{2}{c|}{\textbf{Shadow}} & \multicolumn{2}{c|}{\textbf{Non-boundary}} & \multicolumn{2}{c|}{\textbf{EGSteal-BB}} & \multicolumn{2}{c|}{\textbf{Boundary (XSTEAL)}} & \multicolumn{2}{c}{\textbf{Hybrid (XSTEAL)}} \\
\textbf{Dataset} & \textbf{Expl.} & Fid. & Acc. & Fid. & Acc. & Fid. & Acc. & Fid. & Acc. & Fid. & Acc. \\
\midrule
\multirow{2}{*}{AIDS}
  & GNNExp & \multirow{2}{*}{88.7} & \multirow{2}{*}{75.6} & \multirow{2}{*}{72.0} & \multirow{2}{*}{74.4} & 70.2 & 65.5 & 84.5 & 75.0 & 83.9 & 76.8 \\
  & PGExp  &  &  &  &  & 72.0 & 67.3 & 88.7 & 80.4 & \textbf{90.5} & 79.8 \\
\midrule
\multirow{2}{*}{Letter-low}
  & GNNExp & \multirow{2}{*}{\textbf{89.2}} & \multirow{2}{*}{89.0} & \multirow{2}{*}{\textbf{89.2}} & \multirow{2}{*}{89.0} & 86.2 & 85.8 & 88.5 & 88.3 & \textbf{89.2} & 89.0 \\
  & PGExp  &  &  &  &  & 88.7 & 88.5 & 85.8 & 85.5 & 81.6 & 81.6 \\
\midrule
\multirow{2}{*}{MNIST}
  & GNNExp & \multirow{2}{*}{35.6} & \multirow{2}{*}{26.2} & \multirow{2}{*}{35.0} & \multirow{2}{*}{24.3} & 37.6 & 26.6 & 30.9 & 21.0 & 35.2 & 26.3 \\
  & PGExp  &  &  &  &  & 35.9 & 25.8 & 28.4 & 19.4 & \textbf{39.5} & 28.7 \\
\midrule
\multirow{2}{*}{MUTAG}
  & GNNExp & \multirow{2}{*}{\textbf{100.0}} & \multirow{2}{*}{80.0} & \multirow{2}{*}{85.0} & \multirow{2}{*}{65.0} & 85.0 & 75.0 & \textbf{100.0} & 80.0 & \textbf{100.0} & 80.0 \\
  & PGExp  &  &  &  &  & 95.0 & 75.0 & \textbf{100.0} & 80.0 & \textbf{100.0} & 80.0 \\
\midrule
\multirow{2}{*}{PTC\_FM}
  & GNNExp & \multirow{2}{*}{73.3} & \multirow{2}{*}{60.0} & \multirow{2}{*}{63.3} & \multirow{2}{*}{63.3} & 53.3 & 46.7 & 83.3 & 50.0 & 90.0 & 56.7 \\
  & PGExp  &  &  &  &  & 53.3 & 46.7 & \textbf{93.3} & 53.3 & 86.7 & 53.3 \\
\midrule
\multirow{2}{*}{Synthie}
  & GNNExp & \multirow{2}{*}{26.3} & \multirow{2}{*}{28.9} & \multirow{2}{*}{23.7} & \multirow{2}{*}{25.0} & 28.9 & 28.9 & 19.7 & 19.7 & \textbf{32.9} & 36.8 \\
  & PGExp  &  &  &  &  & 28.9 & 28.9 & 21.1 & 19.7 & 21.1 & 25.0 \\
\midrule
\multirow{2}{*}{NCI1}
  & GNNExp & \multirow{2}{*}{90.2} & \multirow{2}{*}{72.1} & \multirow{2}{*}{89.2} & \multirow{2}{*}{71.4} & 86.2 & 69.8 & 77.6 & 63.2 & \textbf{91.4} & 72.1 \\
  & PGExp  &  &  &  &  & 84.5 & 67.0 & 90.4 & 72.0 & 91.0 & 71.4 \\
\midrule
\multirow{2}{*}{Tox21}
  & GNNExp & \multirow{2}{*}{85.9} & \multirow{2}{*}{74.4} & \multirow{2}{*}{50.0} & \multirow{2}{*}{59.0} & 55.1 & 56.4 & \textbf{94.9} & 83.3 & \textbf{94.9} & 80.8 \\
  & PGExp  &  &  &  &  & 50.0 & 59.0 & 80.8 & 71.8 & 93.6 & 79.5 \\
\bottomrule
\end{tabular}%
}
\caption{Extraction results on \textit{GraphSAGE} victim at matched training set sizes. Fid.\ = Fidelity (\%), Acc.\ = Accuracy (\%). Best fidelity per dataset in \textbf{bold}.}
\label{tab:main-sage}
\end{table}

\begin{figure}[H]
    \centering
    \includegraphics[width=\textwidth]{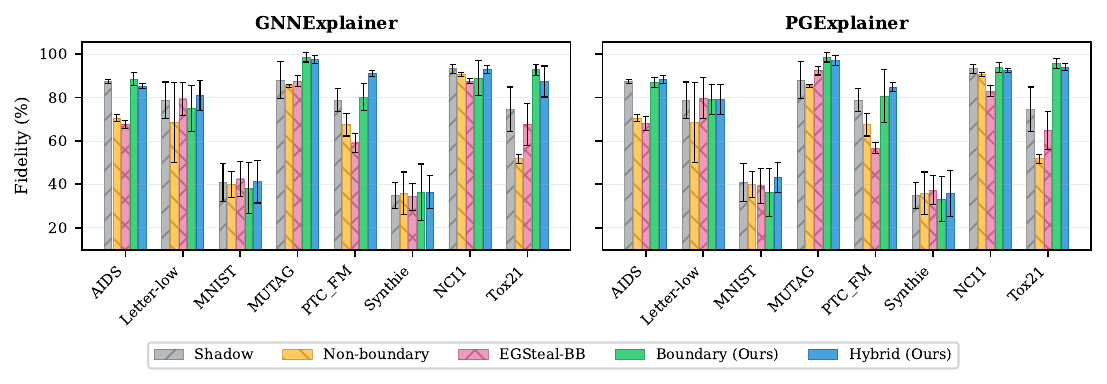}
    \caption{Extraction fidelity (\%) across all architectures and datasets, grouped by explainer (left: GNNExplainer, right: PGExplainer).}
    \label{fig:bar-fidelity}
\end{figure}

\subsection{Monte Carlo Sampling Ablation}
\label{app:component-ablations}

Table~\ref{tab:component-ablations} isolates the contribution of Monte Carlo sampling. The contribution of explanation guidance is isolated separately by the matched No-Exp control (Appendix~\ref{app:matched-noexp}).

\begin{table}[H]
\centering
\small
\begin{tabular}{@{}l|l|cc|cc|cc@{}}
\toprule
\textbf{Explainer} & \textbf{Method} & \multicolumn{2}{c|}{\textbf{AIDS}} & \multicolumn{2}{c|}{\textbf{NCI1}} & \multicolumn{2}{c}{\textbf{Tox21}} \\
 & & Fid. & Acc. & Fid. & Acc. & Fid. & Acc. \\
\midrule
\multirow{2}{*}{GNNExp.} & Boundary (XSTEAL) &{91.1} & {75.9} &{96.0} &{66.3} & {89.3} & {83.3} \\
 & No-MC & {77.6} & {68.1} & {94.7} &{66.5} & {89.3} & {78.6} \\
\bottomrule
\end{tabular}
\caption{Component ablation at 70\% budget (GCN victim). No-MC removes Monte Carlo sampling.}
\label{tab:component-ablations}
\end{table}

\subsection{Matched Queries No-Exp Control}
\label{app:matched-noexp}
To isolate the contribution of explanation guidance from raw computational budget, we compare Boundary against a No-Exp control that is identical except that candidate edges are drawn uniformly at random from all node pairs rather than from the explanation subgraph (Eq.~\ref{eq:explain-candidates}). Both methods run at the default attack hyperparameters with candidate-set size matched per graph and per iteration, so the two methods differ only in which node pairs are scored. We report fidelity both at matched training-set size $n_{\mathrm{train}}$ and at equal total victim query count $Q$.
\begin{table}[H]
\centering
\small
\begin{tabular}{@{}ll|c|c|c@{}}
\toprule
\textbf{Dataset} & \textbf{Budget} & \textbf{Boundary} & \textbf{No-Exp} & \textbf{Hybrid} \\
\midrule
\multirow{2}{*}{PTC\_FM}
 & matched $n_{\mathrm{train}}=60$   & 100.0 & 89.3 & 82.1 \\
 & matched $Q=24{,}184$              & 100.0 & 89.3 & 92.9 \\
\midrule
\multirow{2}{*}{NCI1}
 & matched $n_{\mathrm{train}}=560$  & 96.0  & 95.7 & 96.5 \\
 & matched $Q=203{,}150$             & 95.7  & 95.6 & 96.9 \\
\midrule
\multirow{2}{*}{Tox21\_AhR}
 & matched $n_{\mathrm{train}}=100$  & 96.4  & 85.7 & 82.1 \\
 & matched $Q=609{,}076$             & 96.4  & 85.7 & 100.0 \\
\bottomrule
\end{tabular}
\caption{Matched No-Exp control. Fidelity (\%). GCN victim, GNNExplainer, one seed, class-balanced test sets.}
\label{tab:matched-noexp}
\end{table}
Explanation guidance improves fidelity over the matched control on all three datasets under both accountings, with margins of $+10.7$, $+0.3$, and $+10.7$ points on PTC\_FM, NCI1, and Tox21\_AhR at matched $n_{\mathrm{train}}$. The near-tie on NCI1 is expected: its single-edge deletion yield in Phase~1 is the highest of the three so Phase~2 (only stage using explanations) supplies little of the training set and there is little room for guidance to help. On PTC\_FM and Tox21\_AhR, Phase~1 yields far fewer boundary pairs and Phase~2 does most of the work where guidance matters. This also explains the low Phase-2 success rate on Tox21\_AhR: the fraction of attempted graphs that yields a boundary pair is $0.098$ against $0.686$ on PTC\_FM and $0.590$ on NCI1, reflecting how scarce boundary pairs are on Tox21\_AhR rather than a weakness of guidance.

\subsection{Phase Ablation}
\label{sec:phase-ablation}

Table~\ref{tab:phase-ablation} compares Phase~1 alone (training only on $\delta_e = 1$ boundary pairs, capped by boundary availability) against Phase~2 alone (explanation-guided MC search on all shadow graphs) on the GCN victim for three datasets. Phase~2 outperforms Phase~1.

\begin{table}[H]
\centering
\small
\begin{tabular}{@{}l|cc|cc@{}}
\toprule
& \multicolumn{2}{c|}{\textbf{Boundary}} & \multicolumn{2}{c}{\textbf{Hybrid}} \\
\textbf{Dataset} & P1 & P2 & P1 & P2 \\
\midrule
AIDS     & 80.2 & \textbf{85.3} & 85.3 & \textbf{91.4} \\
NCI1     & 94.9 & \textbf{95.2} & 93.8 & \textbf{95.1} \\
Tox21\_AhR    & 71.4 & \textbf{96.4} & 85.7 & \textbf{92.9} \\
\bottomrule
\end{tabular}
\vspace{0.2em}
\caption{Phase ablation (GCN + GNNExp): Fidelity (\%) for Phase~1 alone vs.\ Phase~2 alone.}
\label{tab:phase-ablation}
\end{table}

\subsection{EGSteal-BB Failure Analysis}
\label{app:egsteal-case-study}
In our general evaluation we noticed that the performance of EGSteal-BB sometimes catastrophically fails, even worse than Shadow. After further analysis, we conclude that this fail is caused by two reasons. First we note the EGSteal-BB is an adapted version for fitting our black-box model settings, and therefore removed with a few components of it which relies on the white-box settings, largely reducing its applied performance. This performance decrease validates our statement on its lack of generalization on black-box settings. Besides, we also find a partial reason, the lack of generalizability of the stated assumption on the causal relationship between the original graph and the generated graph. Specifically, in EGSteal-BB, a generated graph containing the explanation subgraph of the original graph, is assumed to always hold the same prediction as the original graph, since the subgraph is the main cause for the prediction. However, this causal assumption is based on 100\% belief on a extremely accurate graph explainer, which is unlikely to happen in practice. To validate this statement, we pick out the generated graphs and compare the fidelity of its assigned label, i.e., if the victim model truly follow these labels. The results are provided as below:
\begin{table}[]
    \centering
    \begin{tabular}{c|c|c|c|c|c|c|c|c}
        \toprule
         Data.&  AIDS& MUTAG&NCI1&Tox21&PTC&Letter&Syn&MNIST\\
         \midrule
         Fidelity& 99.8\%&96.8\%&96.3\%&99.9\%&100.0\%&96.7\%&99.1\%&93.3\%\\
         \bottomrule
    \end{tabular}
    \caption{Fidelity of the generated training data by EGSteal-BB.}
    \vspace{-2em}
    \label{tab:egsteal-genfidelity}
\end{table}
We additionally note that all other baselines, including our method, are trained on 100\% fidelity sampled data on all datasets. Therefore, these not-fully-faithful generated data add more fidelity loss at the training stage, which leads to its worse performance compared with other baselines.

\section{Fidelity vs.\ Query Budget Trajectories}
\label{app:trajectory}

\begin{figure}[H]
    \centering
    \includegraphics[width=\textwidth]{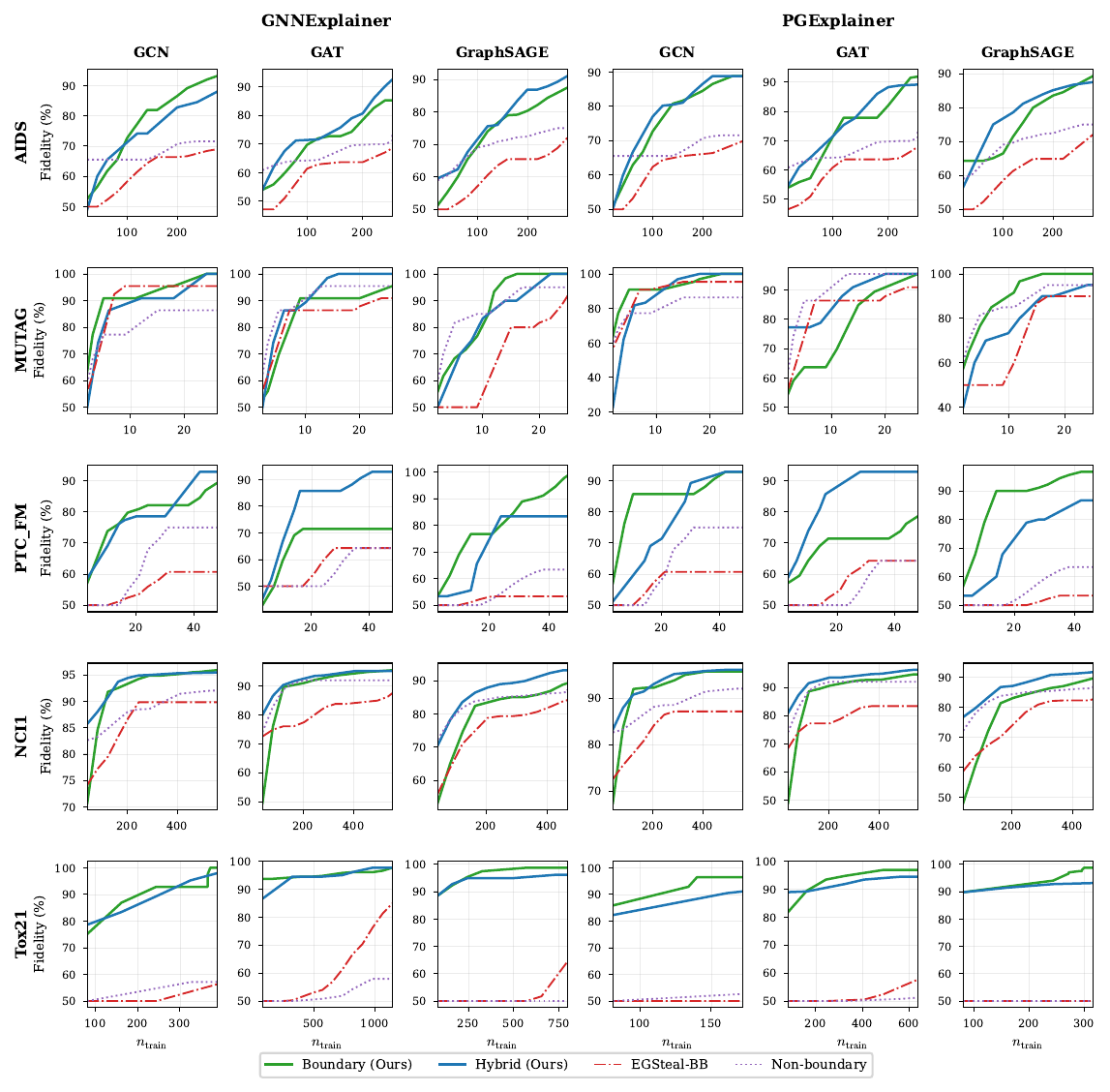}
    \caption{Fidelity (\%) vs.\ training set size across binary datasets, victim architectures, and explainers. Rows correspond to datasets; columns to victim architectures. Top half: GNNExplainer; bottom half: PGExplainer.}
    \label{fig:fidelity-trajectory}
\end{figure}

Figure~\ref{fig:fidelity-trajectory} shows surrogate fidelity as a function of training set size ($n_{\text{train}}$) across all binary datasets, three victim architectures (GCN, GAT, GraphSAGE), and both explainers (top: GNNExplainer, bottom: PGExplainer).

\paragraph{Setup.} For each dataset-architecture-explainer combination, we evaluate all methods at 14 training budgets from 5\% to 70\% of the shadow set in 5\% increments. At each budget level, $n_{\text{train}} = \lfloor |\mathcal{G}_{\text{shadow}}| \times \text{budget} \rfloor$ samples are drawn per method. Non-boundary does not use explanations; EGSteal-BB, Boundary, and Hybrid are evaluated under both GNNExplainer and PGExplainer. Fidelity trajectories apply a light smoothing (window size 3) for visualization clarity. All plots are x-axis-aligned: for each dataset-architecture-explainer triplet, the x-axis range is clipped to $[\max(\min n_{\text{train}}), \min(\max n_{\text{train}})]$ across methods to ensure fair visual comparison.

Several patterns emerge. First, Boundary (green) and Hybrid (blue) consistently dominate across training budgets, with the gap most pronounced at small $n_{\text{train}}$ as boundary pairs provide high-information samples that are especially valuable in low-data regimes. Second, on Tox21 (bottom row), the separation is dramatic: Boundary and Hybrid rise steeply while EGSteal-BB (red dashed) and Non-boundary (gray dotted) remain flat near chance, reflecting Tox21's low 4.2\% boundary rate where Phase~2's generative capacity is essential. Third, EGSteal-BB exhibits high variance and sometimes degrades with more data (e.g., AIDS/GAT, PTC\_FM/GCN), consistent with label noise from imprecise explanation masks compounding as augmented samples accumulate. Fourth, the patterns are consistent across GNNExplainer (top) and PGExplainer (bottom), confirming our method's robustness to explainer choice.

\section{Query Complexity Analysis}
\label{sec:query-complexity}

We analyze the number of black-box queries required to transform one initial
graph $G$ into a boundary pair $(G,G')$. Let $n=|V(G)|$, let $k$ denote the
number of edges flipped per update step, let $n_e$ denote the number of Monte
Carlo samples used to estimate each candidate edge sensitivity, and let
$\Delta_e$ be the maximum allowed edit distance between the initial graph and
the returned boundary graph. At iteration $t$, let
\[
    C_t = C(G^{(t)})
\]
be the explanation-guided candidate edge set, and define
\[
    C_{\max}
    =
    \max_{0\leq t<T_G} |C_t|,
\]
where $T_G$ is the number of update iterations needed before the algorithm
finds a boundary graph.

Since each iteration flips at most $k$ edges, the number of update iterations
is bounded by
\[
    T_G
    \leq
    \min\left\{
        T_{\max},
        \left\lceil \frac{\Delta_e}{k} \right\rceil
    \right\}.
\]
The second term is conditional: $\lceil \Delta_e / k \rceil = \lceil 20/3 \rceil = 7$ holds only if every iteration increases the edit distance by exactly $k$, which requires that no update revisits an already-flipped edge. Because the carry-over set $\mathcal{S}_{\mathrm{carry}}$ can reintroduce such an edge and edge flips are involutions, we rely on the unconditional bound $T_G \le T_{\max} = 10$.
More generally, if the first boundary crossing found by the search occurs at
edit distance $d_B(G)\leq \Delta_e$ from the initial graph, then
\[
    T_G
    \leq
    \left\lceil \frac{d_B(G)}{k} \right\rceil .
\]
Thus, the step size $k$ directly controls the number of search iterations:
larger $k$ reduces the number of rounds, while smaller $k$ yields finer-grained
boundary localization.

At each Phase-2 iteration, the algorithm issues a single API call on $G^{(t)}$ that
returns both the prediction $y(G^{(t)})$ and the explanation mask
$G_{\mathrm{sub}}^{(t)}$ together consistent with our threat model, followed by
Monte Carlo sensitivity estimation for all edges in $C_t$. For each
candidate edge $e\in C_t$, each Monte Carlo sample requires two victim
prediction queries, one for $\widetilde G_m^{(t)}$ and one for
$\widetilde G_m^{(t)}\oplus e$. Therefore, the number of victim prediction
queries in one Phase-2 iteration is
\[
    1 + 2 n_e |C_t|,
\]
where the leading term is the single label-and-mask call, so the explanation
incurs no separately-charged query.

Therefore, for a single graph entering Phase 2, the total number of victim
prediction queries is bounded by
\[
\begin{aligned}
    Q_{\mathrm{pred}}^{(2)}(G)
    &\leq
    1 + \sum_{t=0}^{T_G-1}
    \left(1 + 2 n_e |C_t|\right) \\
    &\leq
    1 + T_G\left(1 + 2 n_e C_{\max}\right) \\
    &=
    O\!\left(
        T_G n_e C_{\max}
    \right).
\end{aligned}
\]
Substituting the step-size bound gives
\[
    Q_{\mathrm{pred}}^{(2)}(G)
    =
    O\!\left(
        n_e C_{\max}
        \min\left\{
            T_{\max},
            \left\lceil \frac{\Delta_e}{k} \right\rceil
        \right\}
    \right).
\]
If the successful boundary distance is $d_B(G)$, this can be refined to
\[
    Q_{\mathrm{pred}}^{(2)}(G)
    =
    O\!\left(
        n_e C_{\max}
        \left\lceil \frac{d_B(G)}{k} \right\rceil
    \right).
\]
The explanation masks arrive with these calls rather than through separate queries: exactly $T_G$ of the prediction queries above also return a mask, so
the number of explanation responses is
\[
    Q_{\mathrm{exp}}^{(2)}(G)
    =
    T_G
    \leq
    \min\left\{
        T_{\max},
        \left\lceil \frac{\Delta_e}{k} \right\rceil
    \right\},
\]
and adds no cost beyond the queries already counted.

Using the explanation-guided candidate construction in Eq.~(9), we have
\[
    C_{\max}
    =
    O\!\left(
        s^2 + s d_{\mathrm{avg}} + c_{\mathrm{carry}} + c_{\mathrm{rand}}
    \right),
\]
where
\[
    s = \max_t |V_{\mathrm{sub}}^{(t)}|,
    \qquad
    c_{\mathrm{carry}} = |S_{\mathrm{carry}}|,
    \qquad
    c_{\mathrm{rand}} = |S_{\mathrm{rand}}|.
\]
Hence the Phase-2 prediction-query complexity becomes
\[
    Q_{\mathrm{pred}}^{(2)}(G)
    =
    O\!\left(
        n_e
        \left(
            s^2 + s d_{\mathrm{avg}}
            + c_{\mathrm{carry}}
            + c_{\mathrm{rand}}
        \right)
        \min\left\{
            T_{\max},
            \left\lceil \frac{\Delta_e}{k} \right\rceil
        \right\}
    \right).
\]
Without explanation-guided reduction, the candidate set could contain all
possible edges, giving $C_{\max}=O(n^2)$ and therefore
\[
    Q_{\mathrm{pred,naive}}^{(2)}(G)
    =
    O\!\left(
        n_e n^2
        \min\left\{
            T_{\max},
            \left\lceil \frac{\Delta_e}{k} \right\rceil
        \right\}
    \right).
\]
Thus, explanation guidance reduces the per-graph Phase-2 query complexity from
quadratic in the graph size to quadratic only in the explanation size.

Including Phase 1, which tests all single-edge removals, the worst-case number
of prediction queries for one initial graph is
\[
    Q_{\mathrm{pred}}(G)
    \leq
    1 + |E(G)|
    +
    \mathbb I_{\mathrm{P2}}(G)
    \cdot
    O\!\left(
        n_e C_{\max}
        \min\left\{
            T_{\max},
            \left\lceil \frac{\Delta_e}{k} \right\rceil
        \right\}
    \right),
\]
where $\mathbb I_{\mathrm{P2}}(G)=1$ if Phase 1 fails to find a boundary pair
and the graph is passed to Phase 2, and $\mathbb I_{\mathrm{P2}}(G)=0$
otherwise. Therefore, in the worst case,
\[
    Q_{\mathrm{pred}}(G)
    =
    O\!\left(
        |E(G)|
        +
        n_e
        \left(
            s^2 + s d_{\mathrm{avg}}
            + c_{\mathrm{carry}}
            + c_{\mathrm{rand}}
        \right)
        \min\left\{
            T_{\max},
            \left\lceil \frac{\Delta_e}{k} \right\rceil
        \right\}
    \right).
\]

This expression makes explicit the role of the step size $k$. Increasing $k$
decreases the number of search iterations roughly as $1/k$, thereby reducing
the query cost. However, larger $k$ also makes the search less local: the
returned boundary pair may lie farther from the initial graph, and the search
may overshoot narrow decision-boundary regions. Conversely, smaller $k$
requires more queries but gives finer control over the final boundary-pair
distance. In practice, $k$ therefore controls a trade-off between query
efficiency and boundary-pair locality.

\section{Selection Consistency of Top-$k$ Edge Selection}
\label{app:selection}
The boundary search (Eq.~\eqref{eq:topk-edges}) uses the estimated sensitivities $\{\hat{g}_e^{(t)}\}$ only to rank candidates and keep the top $k$, so its behavior is invariant to any estimation error that preserves the order of the leading candidates. The corollary below makes this precise: the estimator need only order edges whose true sensitivities are separated by a margin, and the required sample size depends on that margin rather than on the absolute accuracy bounded in Theorem~\ref{lem:hoeffding}.

\begin{corollary}[Selection consistency]
\label{cor:selection}
Fix iteration $t$ and two candidates $e, e' \in \mathcal{C}(G^{(t)})$ with $g_e^{(t)} - g_{e'}^{(t)} \ge \Delta > 0$. With $Z \in [-1,1]$ and $n_e$ Monte Carlo rounds per edge,
\[
    \Pr\!\left( \hat{g}_{e'}^{(t)} \ge \hat{g}_e^{(t)} \right)
    \;\le\;
    \exp\!\left( -\frac{n_e \Delta^2}{8} \right).
\]
Consequently, for confidence $1-s$, if
\[
    n_e \;\ge\; \frac{8}{\Delta^2}\,\log\!\frac{\binom{|\mathcal{C}(G^{(t)})|}{2}}{s},
\]
then with probability at least $1-s$ every pair of candidates whose true sensitivities differ by at least $\Delta$ is ordered correctly, and $\operatorname{TopK}_{\max}$ excludes no edge whose true sensitivity exceeds that of a selected edge by more than $\Delta$.
\end{corollary}

\begin{proof}
Let $D = \hat{g}_e^{(t)} - \hat{g}_{e'}^{(t)} = \frac{1}{n_e}\sum_{m=1}^{n_e} U_m$ with $U_m = Z_{e,m}^{(t)} - Z_{e',m}^{(t)} \in [-2,2]$, independent across $m$ and with mean $\mathbb{E}[U_m] = g_e^{(t)} - g_{e'}^{(t)} \ge \Delta$. A misordering is the event $D \le 0$, which implies $D - \mathbb{E}[D] \le -\Delta$. Applying Hoeffding's inequality to $\sum_m U_m$ (each of range $4$),
\[
    \Pr(D \le \mathbb{E}[D] - \Delta)
    \;\le\;
    \exp\!\left( -\frac{2 (n_e \Delta)^2}{n_e \cdot 4^2} \right)
    \;=\;
    \exp\!\left( -\frac{n_e \Delta^2}{8} \right).
\]
Taking a union bound over the at most $\binom{|\mathcal{C}(G^{(t)})|}{2}$ candidate pairs and setting the right-hand side to $s / \binom{|\mathcal{C}(G^{(t)})|}{2}$ yields the sample size. The final claim holds because violating the Top-$k$ selection requires some excluded edge to outrank a selected one, i.e., a reversed $\Delta$-separated pair.
\end{proof}

Two consequences follow. A coarse estimator (small $n_e$) already selects correctly whenever the leading edges are well separated in true sensitivity, so accuracy in $\hat{g}_e^{(t)}$ is not needed where it does not change the ranking. And even a misordering never corrupts the output: every recorded pair is checked against the victim by Eq.~\eqref{eq:boundary-stop}, so a wrong ranking costs at most search efficiency, not correctness.

\end{document}